\documentclass[12pt]{article}
\usepackage{amsmath}
\usepackage{graphicx}
\usepackage{enumerate}

\usepackage{natbib} 

\newcommand{\blind}{1}
\newcommand{\ywh}{\hat y^{(rwt-k)}_{t,i}}
\newcommand{\yw}{y_{t,i}^{(rwt-k)}}
\newcommand{\samplet}{\bs_{t,i}^{(k)},a_{t,i}^{(k)}}
\newcommand{\sampleki}{\bs_{t,i}^{(k_i)},a_{t,i}^{(k_i)}}
\newcommand{\omegai}{\omega_{t,i}^{(k)}}
\newcommand{\omegaih}{\hat\omega_{t,i}^{(k)}}

\usepackage{amsmath,amssymb,amsfonts,amsthm,bbm}

\usepackage{mathtools}   
\usepackage{stmaryrd} 

\usepackage{enumitem}   
\usepackage{romannum}

\usepackage[titletoc,title]{appendix}  

\usepackage[table,xcdraw,dvipsnames]{xcolor}   
\usepackage{hyperref}
\newcommand\myshade{85}
\colorlet{mylinkcolor}{YellowOrange}
\colorlet{mycitecolor}{Aquamarine}
\colorlet{myurlcolor}{violet}

\hypersetup{
  linkcolor  = mylinkcolor!\myshade!black,
  citecolor  = mycitecolor!\myshade!black,
  urlcolor   = myurlcolor!\myshade!black,
  colorlinks = true,
}

\usepackage{caption}
\usepackage{subcaption}

\usepackage[normalem]{ulem}
\usepackage{setspace}

\usepackage{graphicx}
\usepackage{comment}
\graphicspath{{figs/}}

\renewcommand{\hat}{\widehat}
\renewcommand{\tilde}{\widetilde}

\newcommand{\bfm}[1]{\ensuremath{\boldsymbol{#1}}} 

\def\ba{\bfm a}     
\def\bb{\bfm b}     
   \def\bC{\bfm C}  
\def\bd{\bfm d}     
     \def\EE{\mathbb{E}}

\def\bh{\bfm h}

   \def\bM{\bfm M}  
     \def\NN{\mathbb{N}}
     
     \def\PP{\mathbb{P}}
     
     \def\RR{\mathbb{R}}
\def\bs{\bfm s}

   \def\bW{\bfm W}  
\def\bx{\bfm x}   \def\bX{\bfm X}  
     
\def\bz{\bfm z}

\def\calA{{\cal  A}} 
\def\calB{{\cal  B}}

\def\calE{{\cal  E}} 
 
\def\calG{{\cal  G}} \def\cG{{\cal  G}}
\def\calH{{\cal  H}}

\def\calL{{\cal  L}} 
\def\calM{{\cal  M}} \def\cM{{\cal  M}}
\def\calN{{\cal  N}} 
\def\calO{{\cal  O}} 
\def\calP{{\cal  P}}

\def\calS{{\cal  S}}

\def\calX{{\cal  X}} 
 
\def\calZ{{\cal  Z}} 

\newcommand{\bfsym}[1]{\ensuremath{\boldsymbol{#1}}}

 \def\btheta{\bfsym {\theta}}

\providecommand{\abs}[1]{\left\lvert#1\right\rvert}
\providecommand{\norm}[1]{\left\lVert#1\right\rVert}

\providecommand{\angles}[1]{\left\langle #1 \right\rangle}
\providecommand{\paren}[1]{\left( #1 \right)}
\providecommand{\brackets}[1]{\left[ #1 \right]}
\providecommand{\braces}[1]{\left\{ #1 \right\}}

\usepackage{mathtools}
\DeclarePairedDelimiterX{\infdivx}[2]{(}{)}{%
  #1 \; \delimsize\| \; #2%
}

\DeclareMathOperator{\Tr}{Tr}

\newcommand*\xbar[1]{%
  \hbox{%
    \vbox{%
      \hrule height 0.4pt 
      \kern0.5ex
      \hbox{%
        \kern-0em
        \ensuremath{#1}%
        \kern-0em
      }%
    }%
  }%
} 
\newtheorem{definition}{Definition}
\newtheorem{assumption}[definition]{Assumption}
\newtheorem{lemma}[definition]{Lemma}

\newtheorem{theorem}[definition]{Theorem}
\newtheorem{corollary}[definition]{Corollary}

\theoremstyle{definition}
\newtheorem{remark}{Remark}

\definecolor{royalpurple}{rgb}{0.47, 0.32, 0.66}
\definecolor{greenfresh}{HTML}{00897B}
\definecolor{bluefresh}{HTML}{1E88E5}
\definecolor{redfresh}{HTML}{E53935}

\definecolor{royalpurple}{rgb}{0.47, 0.32, 0.66}

\def\beq{\begin{equation}}
\def\eeq{\end{equation}}

\def\bet{\begin{theorem}}
\def\eet{\end{theorem}}

\def\bel{\begin{lemma}}
\def\eel{\end{lemma}}

\def\eps{\varepsilon}

\def\cond{\;|\;}

\usepackage{setspace}
\usepackage{etoolbox}
\BeforeBeginEnvironment{equation*}{\begin{singlespace}\vspace*{-\baselineskip}}
	\AfterEndEnvironment{equation*}{\end{singlespace}\noindent\ignorespaces}
\BeforeBeginEnvironment{equation}{\begin{singlespace}\vspace*{-\baselineskip}}
	\AfterEndEnvironment{equation}{\end{singlespace}\noindent\ignorespaces}
\BeforeBeginEnvironment{align}{\begin{singlespace}\vspace*{-\baselineskip}}
	\AfterEndEnvironment{align}{\end{singlespace}\noindent\ignorespaces}
\BeforeBeginEnvironment{align*}{\begin{singlespace}\vspace*{-\baselineskip}}
	\AfterEndEnvironment{align*}{\end{singlespace}\noindent\ignorespaces}
\BeforeBeginEnvironment{eqnarray}{\begin{singlespace}\vspace*{-\baselineskip}}
	\AfterEndEnvironment{eqnarray}{\end{singlespace}\noindent\ignorespaces}
\BeforeBeginEnvironment{eqnarray*}{\begin{singlespace}\vspace*{-\baselineskip}}
	\AfterEndEnvironment{eqnarray*}{\end{singlespace}\noindent\ignorespaces}

\usepackage[framemethod=TikZ]{mdframed} 

\usepackage{fancybox}

\usepackage[linesnumbered,ruled,vlined]{algorithm2e}

\SetCommentSty{mycommfont}

\SetKwInput{KwInput}{Input}                
\SetKwInput{KwOutput}{Output}              
\SetKwInput{KwInitialize}{Initialize}    

\usepackage{listings}             
\begin{document}
\pagenumbering{arabic}

\def\spacingset#1{\renewcommand{\baselinestretch}%
{#1}\small\normalsize} \spacingset{1}
\def\r#1{\textcolor{red}{\bf #1}}
\def\b#1{\textcolor{blue}{\bf #1}}
\newcommand{\jcadd}[1]{\noindent{\textcolor{blue}{\{\em #1\}}}}

%
%
%

\def\TITLE{Deep Transfer $Q$-Learning for Offline Non-Stationary Reinforcement Learning}

\if1\blind
{
\title{\bf \TITLE}
\author{
Jinhang Chai$^\flat$ \hspace{8ex}
Elynn Chen$^\natural$ \hspace{8ex}
Jianqing Fan$^\sharp$ \thanks{Corresponding author.
The authors gratefully acknowledge the research support from NSF Grants DMS-2210833 and DMS-2053832,  ONR  N00014-22-1-2340, and DMS-2412577.
}
\\ \normalsize
\medskip
$^{\flat,\sharp}$ Princeton University \hspace{8ex}
$^{\natural}$ New York University
}
\maketitle
} \fi

\if0\blind
{
\bigskip
\bigskip
\bigskip
\begin{center}
{\LARGE\bf \TITLE}
\end{center}
\medskip
} \fi

\bigskip
\begin{abstract}
\spacingset{1.08}
In dynamic decision-making scenarios across business and healthcare, leveraging sample trajectories from diverse populations can significantly enhance reinforcement learning (RL) performance for specific target populations, especially when sample sizes are limited. While existing transfer learning methods primarily focus on linear regression settings, they lack direct applicability to reinforcement learning algorithms.
This paper pioneers the study of transfer learning for dynamic decision scenarios modeled by non-stationary finite-horizon Markov decision processes, utilizing neural networks as powerful function approximators and backward inductive learning. We demonstrate that naive sample pooling strategies, effective in regression settings, fail in Markov decision processes.
To address this challenge, we introduce a novel {\it ``re-weighted targeting procedure''} to construct {\it ``transferable RL samples''} and propose {\it ``transfer deep $Q^*$-learning''}, enabling neural network approximation with theoretical guarantees. We assume that the reward functions are transferable and deal with both situations in which the transition densities are transferable or nontransferable.  Our analytical techniques for transfer learning in neural network approximation and transition density transfers have broader implications, extending to supervised transfer learning with neural networks and domain shift scenarios.
Empirical experiments on both synthetic and real datasets corroborate the advantages of our method, showcasing its potential for improving decision-making through strategically constructing transferable RL samples in non-stationary reinforcement learning contexts.
\end{abstract}

\noindent%
\vfill


\newpage
\spacingset{1.9} 

\addtolength{\textheight}{.1in}%

\section{Introduction}  \label{sec:intro}

Sequential decision-making problems in healthcare, education, and economics are commonly modeled as finite-horizon MDPs and solved using reinforcement learning (RL)  \citep{schulte2014q,charpentier2021reinforcement}. 
These domains face challenges from high-dimensional state spaces and limited data in new contexts. This motivates the development of knowledge transfer that can leverage data from abundant source domains to improve decision-making in target populations with scarce data.

Transfer learning  has shown promises in addressing these challenges, but its effective application to RL remains limited. 
Although transfer learning has advanced significantly in regression settings \citep{li2022transfer-jrssb,gu2022robust,fan2023robust}, these methods do not readily extend to RL problems. 
Recent empirical work on deep RL transfer has focused on game environments \citep{zhu2023transfer}, but their assumptions -- such as identical source-target tasks, predefined reward differences, or known task mappings -- are too restrictive for real-world applications.
While theoretical advances have emerged for model-based transfer in linear low-rank and stationary MDPs \citep{agarwal2023provable,bose2024offline}, a comprehensive theory for transfer learning in non-stationary model-free RL remains elusive. 

To address the limitations in current literature, this paper presents a theoretical study of transfers in non-stationary finite-horizon MDPs, a crucial model within RL. 
Our rigorous analysis in Section \ref{sec:model} reveals {\it fundamental differences} between transfer learning in RL and regression settings. 
Unlike single-stage regression, RL involves multi-stage processes with state transitions, necessitating consideration of state drift. Moreover, RL's delayed rewards, absent in regression settings, require estimation at decision time, introducing additional complexity to the transfer learning process.

We demonstrate that naive pooling of sample trajectories, effective in regression transfer learning, leads to uncontrollable bias in RL settings. To overcome this, we focus on non-stationary finite-horizon MDPs and introduce a novel ``re-weighted targeting procedure'' for {\it backward inductive $Q^*$-learning} \citep{murphy2005generalization,clifton2020q} with neural network function approximation in offline learning. 
This procedure, comprising re-weighting and re-targeting steps, addresses transition shifts and reward prediction misalignments, respectively.

Our work establishes theoretical guarantees for transfer learning in this context, extending insights to deep transfer learning more broadly. We also introduce a neural network estimator for transition probability ratios, also contributing to the study of domain shift in deep transfer learning.
Our {\it contributions} span four key areas. First, we clarify the fundamental differences between transfer learning in RL and that in regression settings, introducing a novel method to construct ``transferable RL samples.'' Second, we develop the ``re-weighted targeting procedure'' for non-stationary MDPs in backward inductive $Q$-learning, which can potentially extend to other RL algorithms. Third, we provide theoretical guarantees for transfer learning with backward inductive deep $Q$-learning, addressing important gaps in the analysis of transfer deep learning and density ratio estimation. Finally, we present a novel mathematical proof for neural network analysis that has broader applications in theoretical deep learning studies. Those include the consideration of temporal dependence in the error propagation in RL with continuous state spaces, removal of the completeness assumption on function class in neural network approximation, and non-asymptotic bounds for density ratio estimator.

\subsection{Related Works and Distinctions of this Work}

This paper bridges transfer learning, statistical RL, and their intersection. We provide a focused review of the most pertinent literature to contextualize our contributions within these interconnected fields.

\smallskip
\noindent
\textbf{Offline RL and finite-horizon $Q$-learning.}
The field of RL is well-documented \citep{sutton2018reinforcement,kosorok2019precision}.
We focus on {\it model-free, offline RL}, distinct from model-based \citep{yang2019sample,li2024settling} and online approaches \citep{jin2023provably}. 
Within $Q$-learning \citep{clifton2020q}, recent work distinguishes between $Q^\pi$-learning for policy evaluation \citep{shi2022statistical} and $Q$-learning for policy optimization \citep{clifton2020q,li2024q}. 

We study finite-horizon $Q$-learning for non-stationary MDPs, building on seminal work on the backward inductive $Q$-learning \citep{murphy2003optimal,murphy2005generalization}. This setting has been explored using linear \citep{chakraborty2014dynamic,laber2014dynamic,song2015penalized} and non-linear models \citep{laber2014interactive,zhang2018interpretable}. However, these studies focus on single-task learning, and to our knowledge, {\it no work} has considered deep neural network approximation in non-stationary finite-horizon $Q$-learning within a transfer learning context.

Machine learning research has primarily concentrated on {\it stationary} MDPs \citep{xia2024instance,li2024q,li2021sample,liao2022batch}. Theoretical advances have emerged in iterative $Q$-learning, particularly under linear MDP assumptions and finite state-space settings \citep{jin2021pessimism,shi2022pessimistic,yan2023efficacy,li2024settling}.
The field has recently expanded to neural network-based approaches. Notable works include \cite{fan2020theoretical}'s analysis of iterative deep $Q$ learning, \cite{yang2020bridging}'s investigation of neural value iteration, and \cite{cai2024neural}'s examination of neural temporal difference learning in stationary MDPs.
Our work differs from these prior studies through its focus on {\it non-stationary} MDPs, specifically addressing the challenges of transfer learning in this context.

\smallskip
\noindent
\textbf{Transfer learning in supervised and unsupervised learning.}
Transfer learning addresses various shifts between source and target tasks: marginal shifts, including covariate \citep{ma2023optimally,wang2023pseudo} and label shifts \citep{maity2022minimax}, and conditional shifts involving response distributions. These have been studied in high-dimensional linear regression \citep{li2022transfer-jrssb,gu2022robust,fan2023robust}, generalized linear models \citep{tian2022transfer,li2023estimation}, non-parametric methods \citep{cai2021transfer,cai2022transfer,fan2023robust}, and graphical models \citep{li2022transfer-jasa}.
Our work differs fundamentally from these settings in three ways. First, offline $Q^*$ estimation involves no direct response observations, requiring novel approaches to transfer future estimations. Second, we develop de-biasing techniques for constructing transferable samples, uniquely necessary in RL contexts. Third, our theoretical analysis of neural network transfer in RL reveals previously unidentified phenomena. Section \ref{sec:model} details these contributions and their implications for transfer learning in RL.

\smallskip
\noindent
\textbf{Transfer learning in RL.}
While transfer learning is well-studied in supervised learning \citep{pan2009survey}, its application to RL poses unique challenges within MDPs. A recent survey \citep{zhu2023transfer} documents diverse empirical approaches in transfer RL, but these often lack theoretical guarantees.
Theoretical advances in transfer RL have primarily focused on model-based approaches with low-rank MDPs \citep{agarwal2023provable,ishfaq2024offline,bose2024offline,lu2021power,cheng2022provable} or stationary model-free settings \citep{chen2024data}. While \cite{chen2022transferred} studied non-stationary environments, they assumed linear $Q^*$ functions and identical state transitions across tasks.

\subsection{Organization}
The paper proceeds as follows. Section \ref{sec:model} formulates transfer learning in non-stationary finite-horizon sequential decisions, defining task discrepancy for MDPs. Section \ref{sec:esti} presents batched $Q$ learning with knowledge transfer using deep neural networks. Section \ref{sec:theory} provides theoretical guarantees under both settings of transferable and non-transferable transition. Section \ref{sec:empirical} presents empirical results. Proofs and additional details appear in the supplementary material.



\section{Transfer RL and Transferable Samples} \label{sec:model}

We consider a non-stationary episodic RL task modeled by a finite-horizon MDP, defined as a tuple $\calM = \braces{\calS, \calA, P, r, \gamma, T}$. Here, $\calS$ is the state space, $\calA$ is the finite action space, $P$ is the transition probability, $r$ is the reward function, $\gamma\in[0,1]$ is the discount factor, and $T$ is the finite horizon.

At time $t$, for the $i$-th individual, an agent observes the current state $\bs_{t,i}\in\calS$, chooses an action $a_{t,i}\in\calA$, and transitions to the next state $\bs_{t+1,i}$ according to $p_t\paren{\bs_{t+1,i}|\bs_{t,i}, a_{t,i}}$. The agent receives an immediate reward $r_{t,i}$ with expected value $r_t(\bs_{t,i}, a_{t,i}) =\EE[r_{t,i}|\bs_{t,i},a_{t,i}]$.

An agent's decision-making is  govern by a policy function $\pi\paren{a_{t,i} | \bs_{t,i}}$ that maps the state space $\calS$ to probability mass functions on the action space $\calA$. For each step $t \in [T]$ and policy $\pi$, we define the state-value function by
\begin{equation} \label{eqn:v-func-0}
    V^\pi_t(\bs) = \EE^{\pi}\brackets{ \sum_{s=t}^T \gamma^{s-t} r(\bs_{s,i}, a_{s,i}) \bigg| \bs_{t,i} = \bs}.
\end{equation}
Accordingly, the action-value function ($Q^\pi$-function) of a given policy $\pi$ at step $t$ is the expectation of the accumulated discounted reward at a state $\bs$ and taking action $a$:
\begin{equation} \label{eqn:$Q^*$-func-0}
    Q^\pi_t(\bs, a) = \EE^{\pi}\brackets{ \sum_{s=t}^T \gamma^{s-t} r(\bs_{s,i}, a_{s,i}) \bigg| \bs_{t,i} = \bs, a_{t,i} = a}.
\end{equation}
For any given action-value function $Q_t^\pi$, 
its greedy policy $\pi^Q_t$ is defined as
\begin{equation} \label{eqn:greedy-policy}
    \pi^Q_t\paren{a | \bs} =
    \begin{cases}
        1 &\text{if } a =  \arg\underset{a'\in\calA}{\max}\; Q^\pi_t\paren{\bs, a'},  \\
        0 & \text{otherwise}.
    \end{cases}
\end{equation}
The goal of RL is to learn an optimal policy $\pi^*_t$, $t\in[T]$, that maximizes the discounted cumulative reward. We define the optimal action-value function $Q^*_t$ as:
\begin{equation} \label{eqn:opt-q}
    Q^*_t\paren{\bs, a} = \underset{\pi\in\Pi}{\sup}\; Q^\pi_t\paren{\bs, a}, \quad \forall \paren{\bs, a} \in \calS \times \calA.
\end{equation}
Then, the Bellman optimal equation holds:
\begin{equation} \label{eqn:bellman-opt}
    \EE\brackets{ r_{t,i} + \gamma\; \underset{a'\in\calA}{\max}\; Q^*_{t+1}\paren{\bs_{t+1,i}, a'} - Q^*_t\paren{\bs_{t,i}, a_{t,i}} \bigg| \bs_{t,i}, a_{t,i} } = 0.
\end{equation}
In non-stationary environments, $Q^*_t$ varies across different stages $t\in[T]$, reflecting the changing dynamics of the system over time.  The optimal policy $\pi^*_t$ can be derived as any policy that is greedy with respect to $Q^*_t$.

For offline RL, backward inductive $Q$ learning has emerged as a classic estimation method. This approach differs from iterative $Q$ learning, which is more commonly used in stationary environments or online settings. Backward inductive $Q$ learning is particularly well-suited for offline finite-horizon problems where the optimal policy may change at each time step, making it an ideal choice for the non-stationary environments with offline estimation. The relationship between these methods and their respective applications are thoroughly explained in the seminal works of \cite{murphy2005generalization} and \cite{clifton2020q}.

\subsection{Transfer Reinforcement Learning}

Transfer RL leverages data from similar RL tasks to enhance learning of a target RL task. We consider source data from offline observational data or simulated data, while the target task can be either offline or online.

Let $[K]$ denote the set of $K$ source tasks. We refer to the target RL task of interest as the $0$-th task, denoted with a superscript ``$(0)$'', while the source RL tasks are denoted with superscripts ``$(k)$'', for $k\in[K]$. For notational simplicity, we sometimes omit the $(0)$ for the target task. For example, $Q^*$ stands for $Q^{*(0)}$.
Random trajectories for the $k$-th source task are generated from the $k$-th MDP $\calM^{(k)} = \braces{\calS, \calA, P^{(k)}, r^{(k)}, \gamma, T}$. We assume, without loss of generality, that the horizon length $T$ is the same for all tasks. For each task $k\in \{0\} \cup [K]$, we collect $n_k$ independent trajectories of length $T$, denoted as $\{\bs_{t,i}^{(k)},a_{t,i}^{(k)},r_{t,i}^{(k)}\}_{t\in[T], i\in[n_k]}$. We assume that trajectories in different tasks are independent and that $n_k$ does not depend on stage $t$, i.e., none of the tasks have missing data. Due to technical reasons for nonlinear aggregation, we further assume $(n_1,n_2,\cdots,n_K)$ follows a multinomial distribution with total number $n_\calM$; see Section \ref{sec:aggregation} for further details.

In single-task RL, each task $k\in \{0\} \cup [K]$ is considered separately. The underlying true response for $Q$-learning at step $t$ is defined as
\begin{equation}  \label{eqn-yk}
y^{(k)}_{t,i}
:= r^{(k)}_{t,i}
+ \gamma \cdot \max_{a\in\calA} Q_{t+1}^{*(k)}(\bs_{t+1,i}^{(k)}, \;a).
\end{equation}
According to the Bellman optimality equation \eqref{eqn:bellman-opt}, we have:
\begin{equation}\label{bellman0}
\EE^{(k)}\brackets{ Y^{(k)}_{t,i} - Q_{t}^{*(k)}\paren{ S^{(k)}_{t,i}, A^{(k)}_{t,i}}
\bigg| S^{(k)}_{t,i},\; A_{t,i}^{(k)} } = 0,
\quad\text{for}\quad k\in \{0\} \cup [K],
\end{equation}
which provides a conditional moment condition for the estimation of $Q^{*(k)}_{t}$.

For the target task, if $y^{(0)}_{t,i}$ were directly observable, $Q^{*(0)}_{t}$ could be estimated via regression via \eqref{bellman0}. However, we only observe the ``partial response'' $r^{(0)}_{t,i}$. The second term on the right-hand side of \eqref{eqn-yk} depends on the unknown $Q^*$-function and future observations. To address this, we employ backward-inductive $Q$-learning, which estimates $Q^{*(0)}_{t}$ in a backward fashion from $t = T$ to $t = 0$, using the convention that $Q_{T+1}^{*(k)}(\bs_{T+1,i}^{(k)}, \;a) \equiv 0$. 
This backward-inductive approach, common in offline finite-horizon $Q$-learning for single-task RL \citep{murphy2005generalization,clifton2020q}, will be extended to the transfer learning setting in subsequent sections.

\subsection{Similarity Characterizations} \label{sec:similarity}

To develop rigorous transfer methods and theoretical guarantees in RL, we need precise definitions of task similarity. We define RL task similarity based on the mathematical model of RL: a tuple $\calM = \braces{\calS, \calA, P, r, \gamma,T}$. Our focus is on differences in reward and transition density functions across tasks. For $k\in \{0\} \cup [K]$, we define the $k$-th RL task as $\calM^{(k)} = \braces{\calS, \calA, P^{(k)}, r^{(k)}, \gamma,T}$, where $k=0$ denotes the target task and $k\in[K]$ denotes source tasks. We characterize similarities as follows:

\smallskip
\noindent
\textbf{(I) Reward Similarity.}
We define the difference between reward functions of the target and the $k$-th source task using the functions 
\begin{equation} \label{eq-delta}
\delta_t^{(k)}(\bs,\; a)
:= r_{t}^{(0)}(\bs,a) - r_{t}^{(k)}(\bs,a) 
\end{equation}
for $t\in[T]$ and $k\in[K]$. Task similarity implies that $\delta^{(k)}_t(\cdot, \cdot)$ is easier to estimate \citep{zhu2023transfer}. Specific assumptions on reward similarity will be detailed in Section \ref{sec:esti}  for neural network and Appendix E for kernel approximations, respectively.

\smallskip
\noindent
\textbf{(II) Transition Similarity.}
We characterize the difference between transition probabilities of the target and the $k$-th source task using the transition density ratio:
\begin{equation}
\omega^{(k)}_t(\bs'|\bs,a) = \frac{p^{(0)}_t(\bs'|\bs,a)}{p^{(k)}_t(\bs'|\bs,a)},
\end{equation}
where $p^{(0)}_t$ and $p^{(k)}_t$ are the transition probability densities of the target and $k$-th source task, respectively. In exploring transition similarity, we examine three distinct scenarios:
\begin{itemize}
    \item Total similarity: $\omega^{(k)}_t(\bs'|\bs,a) = 1$.
    \item Transferable transition densities: $p^{(k)}_t(\bs'|\bs,a)$ are similar so that their ratio, as measured by $\omega^{(k)}_t$, is of lower order of complexity. 
    \item  Non-transferable transition densities:  $p^{(k)}_t(\bs'|\bs,a)$ are so different that there are no advantages of transfer this part of knowledge.
\end{itemize}

\begin{remark}
Our similarity metric based on reward function discrepancy has been  used in empirical studies \citep{zhu2023transfer} and theoretical analyses \citep{chen2022transferred,chen2024data}. It generalizes previous definitions \citep{lazaric2012transfer,mousavi2014context} by allowing similar but different $Q^*$-functions, making it applicable to various domains where responses to treatments may vary slightly. As rewards are directly observable, assumptions on \eqref{eq-delta} can be verified in practice. The scenario of transition density similarity, however, has not been rigorously studied before.
\end{remark}

\begin{remark}
The similarity quantification in \eqref{eq-delta} can be interpreted from a potential outcome perspective. For a given state-action pair $(\bs,a)$, $\delta_t^{(k)}(\bs,a)$ represents the difference in reward when switching from the $k$-th task to the target task. While this ``switching'' describes unobserved counterfactual facts \citep{kallus2020more}, the potential response framework allows us to generate counterfactual estimates using samples from the $k$-th study and estimated coefficients of the target study.
\end{remark}

\subsection{Challenges in Transfer $Q$-Learning}

In RL, unlike supervised learning (SL), the true response $y^{(0)}_{t,i}$ defined in \eqref{eqn-yk} is unavailable. For single-task $Q$-learning, we construct pseudo-responses:

\begin{equation}  \label{eqn-pseudo-y}
    \hat y^{(k)}_{t,i}
    := r^{(k)}_{t,i} + \gamma\cdot\max_{a\in\calA} \hat Q^{(k)}_{t+1}(\bs_{t+1,i}^{(k)}, a), \quad\text{for}\quad k\in\{0\}\cup[K]. 
\end{equation}
A naive extension of transfer learning to RL would augment target pseudo-samples $\{\bs^{(0)}_{t,i}, a^{(0)}_{t,i},\hat y^{(0)}_{t,i}\}$ with source pseudo-samples $\{\bs^{(k)}_{t,i},a^{(k)}_{t,i}, \hat y^{(k)}_{t,i}\}$. However, this introduces additional bias due to the mismatch between source and target $Q^*$ functions:
\begin{equation}
Q^{*(0)}_{t}(\bs, a) - Q^{*(k)}_{t}(\bs, a) = \delta^{(k)}_t(\bs, a) + b_t^{(k)}(\bs,a), 
\end{equation}
where
\begin{equation}
\begin{aligned}
b_t^{(k)}(\bs,a) & = \EE^{(0)}\brackets{\gamma\cdot\max_{a\in\calA} Q^{*(0)}_{t+1}(S_{t+1,i}^{(0)}, a)~\big|~S^{(0)}_{t,i}=\bs, A^{(0)}_{t,i} =a} \\
& - \EE^{(k)}\brackets{\gamma\cdot\max_{a\in\calA} Q^{*(k)}_{t+1}(S_{t+1,i}^{(k)}, a)~\big|~S^{(k)}_{t,i}=\bs, A^{(k)}_{t,i}=a}.
\end{aligned}
\end{equation}
While $\delta^{(k)}_t(\bs, a)$ is unavoidable and can be controlled, $b_t^{(k)}(\bs,a)$ is difficult to validate or learn, making direct application of SL transfer techniques infeasible for RL.

\subsection{Re-Weighted Targeting for Transferable Samples}

We propose a novel ``re-weighted targeting (RWT)'' approach to construct transferable pseudo-responses. At the population level, given the transition density ratio $\omega^{(k)}_t = {p^{(0)}_t}/{p^{(k)}_t}$, we define
\begin{equation}  \label{eq-mt}
y^{(rwt-k)}_{t,i}
:= r^{(k)}_{t,i} + \gamma\cdot\omega^{(k)}_{t,i}\cdot \max_{a\in\calA} Q^{*(0)}_{t+1}(\bs_{t+1,i}^{(k)}, a),
\end{equation}
where $\omega^{(k)}_{t,i} = \omega^{(k)}_{t}(\bs_{t+1,i}^{(k)}\cond\bs_{t,i}^{(k)},a_{t,i}^{(k)})$ and the target $Q^{*(0)}_{t+1}$ is used.

This approach aligns the future state of source tasks with that of the target task. From \eqref{bellman0}, \eqref{eq-delta}, and \eqref{eq-mt}, we easily see that
\begin{align}
&\EE^{(0)}\brackets{ Y^{(0)}_{t,i}~\big|~S^{(0)}_{t,i}=\bs, A^{(0)}_{t,i} =a}
= Q^{*(0)}_{t}(\bs, a)  \label{mm1}\\
&\EE^{(k)}\brackets{ Y^{(rwt-k)}_{t,i}~\big|~S^{(k)}_{t,i}=\bs, A^{(k)}_{t,i}=a}
= Q^{*(0)}_{t}(\bs, a) - \delta^{(k)}_t(\bs, a).  
\label{mm2}
\end{align}
The model discrepancy between $y^{(0)}_{t,i}$ and $y^{(rwt-k)}_{t,i}$ arises solely from the reward function inconsistency at stage $t$.

In practice, we construct pseudo-responses for the target samples:
\begin{equation}  \label{yk-pseudo}
\hat y^{(0)}_{t,i}
:= r^{(0)}_{t,i} +\gamma\cdot\max_{a\in\calA} \hat Q_{t+1}^{(0)}(\bs^{(0)}_{t+1,i},a)
\end{equation}
and RWT pseudo-responses from the source samples:
\begin{equation}  \label{eq-mt-pseudo}
\hat{y}^{(rwt-k)}_{t,i}
:= r^{(k)}_{t,i} + \gamma\cdot\hat\omega_{t,i}^{(k)}\cdot \max_{a\in\calA} \hat Q^{(0)}_{t+1}(\bs_{t+1,i}^{(k)}, a),
\end{equation}
where $\hat Q^{(0)}_{t+1}$ and $\hat\omega_{t,i}^{(k)}$ are estimates of $Q^{*(0)}_{t+1}$ and $\omega_{t,i}^{(k)}$, respectively.

The ``re-weighted targeting (RWT)'' procedure enables ``cross-stage transfer,'' a phenomenon unique to RL transfer learning. Improved estimation of $\hat Q^{(0)}_{t+1}$ using source data at stage $t+1$ enhances the accuracy of RWT pseudo-samples at stage $t$, facilitating information exchange across stages and boosting algorithm performance.
We instantiate the approximation methods and theoretical results under deep ReLU neural networks in Section \ref{sec:esti} and also under kernel approximation in Appendix E. 

\subsection{Aggregated Reward and $Q^*$-Functions}
\label{sec:aggregation}
The implicit data generating process can be understood as randomly assigning a task number $k_j$ to the $j$-th sample with probability $\upsilon_k$ for $1\le j \le n_\calM$ and getting  random samples of sizes $n_0,\cdots,n_K$ for tasks $0, \cdots, K$. Hereafter, we interchangeability use two types of notations. Conditional on realizations of $n_0,\cdots,n_K$, we write the samples as the collections of trajectories  $\{(\bs^{(k)}_{t,i},a^{(k)}_{t,i},\bs'^{(k)}_{t,i})_{t=1}^T\}_{i=1}^{n_k}, \ k\in \{0\}\cup[K]$. We also write $\{(\bs^{(k_j)}_{t,j},a^{(k_j)}_{t,j},\bs'^{(k_j)}_{t,j})_{t=1}^T\}_{j=1}^{n_\calM}$, where $k_j$ is the task number corresponding to the $j$-th sample.
Let $\tilde\mu$ be the product of Lebesgue measure and counting measure on $\calS\times \calA$. For each task $k$, the sample trajectories $\{(\bs^{(k)}_{t,i},a^{(k)}_{t,i},\bs'^{(k)}_{t,i})_{t=1}^T\}_{i=1}^{n_k}$ are i.i.d.~across index $i$.  Thus, at stage $t$,  they follow a joint distribution $\PP_t^{(k)}$, with density $p_t^{(k)}$ with respect to $\tilde\mu$.  

Define $\bar\upsilon_t^{(k)}(\bs,a)=\PP[k_i=k|\bs_{t,i}=\bs,a_{t,i}=a]$ as the conditional probability of task given $\bs,a$ at time $t$.
We then define the aggregated reward function as $r_t^{*\operatorname{agg}} = \sum_{k=1}^K \bar\upsilon_t^{(k)}r_t^{(k)}$ and the aggregated $Q^*$ function as a weighted average 
\begin{equation}\label{eqn:Q-agg}
Q_t^{*\operatorname{agg}} = \sum_{k=1}^K {\bar\upsilon_t^{(k)}}\EE^{(k)}\brackets{Y^{(rwt-k)}_{t,i}~\big|~S^{(k)}_{t,i}=\bs, A^{(k)}_{t,i}=a}.    
\end{equation}
From equations \eqref{mm1} and \eqref{mm2}, we have
\begin{equation}
\sum_{k=1}^K {\bar\upsilon_t^{(k)}} \EE^{(k)}\brackets{ Y^{(rwt-k)}_{t,i}~\big|~S^{(k)}_{t,i}=\bs, A^{(k)}_{t,i}=a}
= Q^{*(0)}_{t}(\bs, a) + \sum_{k=1}^K {\bar\upsilon_t^{(k)}} \delta^{(k)}_t(\bs, a). 
\end{equation}
From Bayes' formula, it holds that  $\bar\upsilon_t^{(k)}(\bs,a)=\frac{\upsilon_k \PP_t^{(k)}(\bs,a)}{\PP_t(\bs,a)}$. Hence we can equivalently write $r^{*\operatorname{agg}}_t(\bs,a)=\frac{\sum_{k=1}^K \upsilon_kp_t^{(k)}(\bs,a)r^{*(k)}_t(\bs,a)}{\sum_{k=1}^K \upsilon_k p_t^{(k)}(\bs,a)}$.
From these definitions, it follows that:
\begin{align}
&Q^{*\operatorname{agg}}_t(\bs,a) - Q^{*(0)}_t(\bs,a) = \delta^{*\operatorname{agg}}_t(\bs,a), \\
&\delta^{*\operatorname{agg}}_t(\bs,a) = r^{*\operatorname{agg}}_t(\bs,a) - r^{(0)}_t(\bs,a) = \sum_{k=1}^K \bar\upsilon_t^{(k)} \delta^{(k)}_t(\bs, a).
\end{align}

The aggregated function $Q^{*\operatorname{agg}}_t(\bs,a)$ represents the optimal $Q^*$ function for a mixture distribution of source MDPs and can be estimated using RWT source pseudo-samples $\{\hat{y}^{(rwt-k)}_{t,i},\bs_{t,i},a_{t,i}\}_{k\in[K]}$. When the similarity condition on $\delta^{(k)}_t(\bs, a)$ is preserved under addition, such as those considered in Section \ref{sec:esti}, $\delta^{*\operatorname{agg}}_t(\bs,a)$ remains correctable using target pseudo-samples -- a crucial population-level property that underlies our estimator construction. In the special case of a single source task ($K=1$), these aggregate function simplify to the source function $Q^{*(1)}$ itself.

\begin{algorithm}[ht!]
\SetKwInOut{Input}{Input}
\SetKwInOut{Output}{Output}
\Input{Target data $\{\{\bs^{(0)}_{t,i}, a^{(0)}_{t,i},r^{(0)}_{t,i}\}_{t\in [T]}\}_{i\in[n_0]}$, \\
source data $\{\{\{\bs^{(k)}_{t,i},a^{(k)}_{t,i},r^{(k)}_{t,i}\}_{t\in[T]}\}_{i\in[n_k]}\}_{k=1}^K$, and \\
discount factor $\gamma\in[0,1]$.}

Let $\hat{Q}^{(0)}_{T+1}(\cdot)=0$ to deal with pseudo-response construction at last stage $T$. 

\tcc{Calculate targeting weights under different conditions of transition similarity.}

\If{Transition Total Similarity}{
Set $\hat\omega_{t}^{(k)}(\bs'|\bs,a) = 1$ for all $t\in[T]$ and $k\in[K]$. 
}

\If{Transition Total Dissimilarity}{
Call {\tt Algorithm \ref{alg:trans-ratio}} to calculate $\{\hat\omega_{t}^{(k)}(\bs'|\bs,a)\}$ for $t\in[T]$ and $k\in[K]$. 
}

\If{Transition Transferable}{
Call {\tt Algorithm \ref{alg:trans-ratio-transfer}} to calculate $\{\hat\omega_{t}^{(k)}(\bs'|\bs,a)\}$ for $t\in[T]$ and $k\in[K]$.
}

\tcc{Calculate $Q^*$-function by backward induction.}

\For{$t=T,\dots, 1$}{ 

\tcc{Constructing transferable RL samples by re-weighted targeting.}
Calculate targeting weights $\hat\omega_{t,i}^{(k)} = \hat\omega_{t}^{(k)}(\bs_{t+1,i}|\bs_{t,i},a_{t,i})$

Construct pseudo-response $\hat{y}^{(rwt-0)}_{t,i}=\hat{y}^{(0)}_{t,i}$ \eqref{yk-pseudo} of the target task.

Construct re-weighted targeting pseudo-response $\hat{y}^{(rwt-k)}_{t,i}$ \eqref{eq-mt-pseudo} using $\hat\omega_{t,i}^{(k)}$. 

\tcc{Supervised regression transfer block: aggregate and debias.}
RWT Transfer $Q$-learning:  
\begin{equation} \label{eqn:trans-q-general}
\begin{aligned}
\hat Q^p_t & := \underset{g\in \calG_1}{\arg\min} \sum_{k\in[K]}\sum_{i\in [n_k]}\left(\hat y^{(rwt-k)}_{t,i}-g(\bs_{t,i}^{(k)},a_{t,i}^{(k)})\right)^2, \\
\hat{\delta}_t & := \underset{g\in \calG_2}{\arg\min} \sum_{i\in [n_0]}\left(\hat y_{t,i}^{(0)}-\hat Q^p_t(\bs_{t,i}^{(0)},a_{t,i}^{(0)})-g(\bs_{t,i}^{(0)},a_{t,i}^{(0)})\right)^2,
\end{aligned}
\end{equation}
where $\cG_1$ and $\cG_2$ are approximating function classes of $Q^*$ and reward difference functions respectively. 

Set $\hat{Q}^{(0)}_t=\hat{Q}^p_t+\hat{\delta}_t$.
}

\Output{$\hat{Q}^{(0)}_t$ for all stage $t \in [T]$.}

\caption{RWT Transfer $Q$-Learning (Offline-to-Offline)}
\label{alg:rwt-trans-q-general}
\end{algorithm}

\subsection{Transfer Backward-Inductive $Q$-Learning}

Based on the preceding discussions, we present the RWT transfer $Q$-Learning in Algorithm \ref{alg:rwt-trans-q-general}.
After the construction of transferable RL samples from lines 2 -- 6, the algorithm's main procedure consists of two steps per stage. 
First, we pool the source and re-targeted pseudo responses to create a biased estimator with reduced variance. Then, we utilize source pseudo responses to correct the bias in the initial estimator. We employ nonparametric least squares estimation for this process. The superscript ``p'' denotes either the ``pooled'' or ``pilot'' estimator, as seen in equation \eqref{eqn:trans-q-general}. 

This algorithm is versatile, as the re-weighted targeting approach for constructing transferable RL samples is applicable to various RL algorithms. Moreover, it allows for flexibility in the choice of approximation function classes $\cG_1$ and $\cG_2$ (the latter is usually a simpler class to make the benefit of the transfer possible).  We implement these classes using Deep ReLU neural networks (NNs) in Section \ref{sec:esti}, with corresponding theoretical guarantees provided in Section \ref{sec:theory}. An alternative instantiation using kernel estimators, along with its theoretical analysis, is detailed in Appendix E of the supplemental material. The methods for estimating weights (line 2 of Algorithm \ref{alg:rwt-trans-q-general}) will be specified for both non-transfer and transfer transition estimation scenarios.

\section{Transfer $Q$-Learning with DNN Approximation} \label{sec:esti}

While our previous discussions on constructing transferable RL samples are applicable to various settings, further algorithmic and theoretical development requires specifying the functional class of the optimal $Q^*$ function and its approximation. In this section, we instantiate RL similarity and the transfer $Q$-learning algorithm using deep ReLU neural network approximation. This approach allows us to leverage the expressive power of neural networks in capturing complex $Q^*$-functions. For an alternative perspective, we present the similarity definition and transfer $Q$-learning algorithm under kernel estimation in Appendix E of the supplemental material.

\subsection{Deep Neural Networks for $Q^*$-Function Approximation.}
Conventional to RL and non-parametric literature, we consider a continuous state space $\calS=[0,1]^d$ and finite action space $\calA=[M]$, which is widely used in clinical trials or recommendation system. 
For the learning function class, we use deep ReLU Neural Networks (NNs). 
Let $\sigma(\cdot)=\max\{\cdot,0\}$ be the element-wise ReLU activation function. Let $L$ be the depth and $\tilde\bd=(d_1,d_2,\cdots,d_{L})$ be the vector of widths. A deep ReLU network mapping from $\RR^{d_0}$ to $\RR^{d_{L+1}}$ takes the form of
\begin{equation}
\label{equ:NN}
g(\bx)=\calL_{L+1}\circ\sigma\circ\calL_L\circ \cdots\circ\calL_2\circ\sigma\circ\calL_1(\bx)\end{equation}
where $\calL_\ell(\bz)=\bW_\ell\bz+\bb_\ell$ is an affine transformation with the weight matrix $\bW_\ell\in \RR^{d_\ell\times d_{\ell-1}}$ and bias vector $\bb_\ell\in \RR^{d_\ell}$.
The function class of deep ReLU NNs is characterized as follows: 
\begin{definition} \label{def:dnn}
Let $L\in\NN$ be the depth, $N\in\NN$ be the width, $B\in\RR$ be the weight bound, and $M\in\RR$ be the truncation level.
The function class of deep ReLU NNs is defined as 
\begin{align*}\calG(L,N,M,B)=\Bigl\{\tilde g(\bx)= T_M(g(\bx)):g\  \text{of form of}\  (\ref{equ:NN})\  \text{with}\  \|\bW_\ell\|_{max},\|\bb_\ell\|_{max}\le B,\\\tilde \bd=(N,N,\cdots,N),d_0=d,d_{L+1}=1\Bigl\}\end{align*}
where $T_M$ is the truncation operator defined as $T_M(z)=\operatorname{sgn}(z)(|z|\wedge M)$.
\end{definition}

\smallskip
\noindent
\textbf{Optimal $Q^*$-function class.}
We assume a general hierarchical composition model  \citep{kohler2021rate} to characterize the low-dimensional structure for the optimal $Q^*$-function:
\begin{definition}[Hierarchical Composition Model]\label{def:hcm}
    The function class of hierarchical composition model (HCM) $\calH(d,l,\calP)$  with $d,l\in\NN^+$ and $\calP$, a subset of $[1,\infty)\times \NN^+$ satisfying $\sup_{\beta,t\in \calP}(\beta\vee t)<\infty$, is defined as follows. For $l=1$, 
\begin{align*}
\calH(d,1,\calP)=\{h:\RR^d\rightarrow \RR:h(\bx)=&g(x_{\tau(1)},\cdots,x_{\tau(t)}),\text{where} \ g:\RR^t\rightarrow\RR \ \text{is}\  (\beta,C)\text{-smooth}\\
&\text{for some} \ (\beta,t)\in\calP\  \text{and} \ \tau:[t]\rightarrow [d]\}
\end{align*}
For $l>1$, $\calH(d,l,\calP)$ is defined recursively as 
\begin{align*}
\calH(d,l,\calP)=\{h:\RR^d\rightarrow \RR:h(\bx)=&g(f_1(\bx),\cdots,f_t(\bx)),\text{where} \ g:\RR^t\rightarrow\RR \ \text{is}\  (\beta,C)\text{-smooth}\\
&\text{for some} \ (\beta,t)\in\calP\  \text{and} \ f_i\in\calH(d,l-1,\calP)\}.
\end{align*}

\end{definition}
Basically, a hierarchical composition model consists of a finite number of compositions of functions with $t$-variate and $\beta$-smoothness for $(\beta,t)\in \calP$.  The difficulty of learning is characterized by the following minmum dimension-adjusted degree of smoothness \citep{fan2024how}:
\[\gamma^*(\calH(d,l,\calP))=\min_{(\beta,t)\in \calP}\frac{\beta}{t}.\]
When clear from the context, we simply write $\gamma^*(\calH)$.

\smallskip
\noindent
\textbf{Reward similarity.}
Intuitively, to ensure statistical improvement with transfer learning, it is necessary that the reward difference $\delta^{(k)}_{t}(\bs,a)$ can be easily learned with even a small number of target samples. 
Under the function classes of the hierarchical composition model and deep neural networks, we directly characterize the easiness of the task difference using aggregated difference defined in Section \ref{sec:aggregation} as follows. 
\begin{assumption}[Reward Similarity] 
\label{asp:task-similarity-1}
For each time $t$, action $a$ and task $k$, we have that $Q^{*\operatorname{agg}}_t(\cdot,a)\in \calH_1$, and $\delta^{*\operatorname{agg}}_{t}(\cdot,a)\in \calH_2$. Further, $\gamma_1=\gamma^*(\calH_1),\gamma_2= \gamma^*(\calH_2)$ satisfy $\gamma_1 <  \gamma_2$.
\end{assumption}

\smallskip
\noindent
\textbf{Transfer deep $Q$-learning.}
Algorithm \ref{alg:rwt-trans-q-general} presents the offline-to-offline transfer deep $Q$-learning with the following specification of function classes:
\begin{equation}\label{eqn:Q^p_t}
\begin{aligned}
\hat Q^p_t & := \underset{g\in \calG(L_1,N_1,M_1,B_1)}{\arg\min} \sum_{k\in\{0\}\cup[K]}\sum_{i\in [n_k]}\left(\hat y^{(rwt-k)}_{t,i}-g(\bs_{t,i}^{(k)},a_{t,i}^{(k)})\right)^2, \\
\hat{\delta}_t & := \underset{g\in \calG(L_2,N_2,M_2,B_2)}{\arg\min} \sum_{i\in [n_0]}\left(\hat y_{t,i}^{(0)}-\hat Q^p_t(\bs_{t,i}^{(0)},a_{t,i}^{(0)})-g(\bs_{t,i}^{(0)},a_{t,i}^{(0)})\right)^2,
\end{aligned}
\end{equation}
where $g\in\calG(L,N,M,B)$ is the deep NN function class in Definition \ref{def:dnn}.

For total similarity $\omega^{(k)}_t(\bs'|\bs,a) = 1$, Algorithm \ref{alg:rwt-trans-q-general} is adequate by plugging in the identity density ratio. 
In the following sections, we proposed methods to estimate the transition density ratio for settings of total dissimilarity $\omega^{(k)}_t(\bs'|\bs,a) \ne 1$ and similarity assumption on $\omega^{(k)}_t(\bs'|\bs,a)$ where transition density transfers is possible.



\subsection{Transition Ratio Estimation without Transition Transfer} \label{sec:transition-ratio-estimation-total-dissimilarity}

We now address the estimation of the transition ratio $\omega^{(k)}_t(\bs'|\bs,a) = \frac{p^{(0)}_t(\bs'|\bs,a)}{p^{(k)}_t(\bs'|\bs,a)}$ under non-transferability of the transition probabilities. 
While density ratio methods with provable guarantees exist \citep{nguyen2010estimating,kanamori2012statistical}, the estimation of $\omega_t^{(k)}$, a {\em ratio of two conditional densies}, presents unique challenges.
Firstly, direct application of M-estimation methods is inadequate for $\omega_t^{(k)}$, as conditional densities lack a straightforward sample version, unlike unconditional densities. Secondly, existing density ratio estimation techniques provide only asymptotic bounds. In contrast, our context requires non-asymptotic bounds for density ratio estimator, a requirement not met by current methods. These challenges collectively necessitate the development of novel approaches to effectively estimate transition ratios in our setting. 

A key insight is that both conditional densities $p^{(0)}_t(\bs'|\bs,a)$ and $p^{(k)}_t(\bs'|\bs,a)$ can be expressed as ratios of joint density to marginal density:
$p^{(k)}_t(\bs'|\bs,a) = \frac{p^{(k)}_t(\bs',\bs,a)}{p^{(k)}_t(\bs,a)}$ for $k\in\{0\}\cup[K]$. 
This formulation allows for separate estimation of each density using M-estimator techniques \citep{nguyen2010estimating,kanamori2012statistical}.
We propose to estimate $p^{(0)}_t(\bs'|\bs,a)$ and $p^{(k)}_t(\bs'|\bs,a)$ independently, leveraging this density ratio representation, and establish non-asymptotic bounds.
This approach provides a pathway to overcome the challenges in directly estimating the ratio of conditional probabilities.

\smallskip \noindent \textbf{Estimating conditional transition density.}
We begin by estimating the function $\rho^{(k)}_t(\bs,a,\bs') := p^{(k)}_t(\bs'|\bs,a)$, which represents the conditional density of the next state given the current state and action. 
Algorithm \ref{alg:marginal-density-dnn} outlines the process of estimating this conditional transition probability using deep neural network approximation. 

For clarity, we present the setting with both $k$ (task) and $t$ (time step) fixed.
The data corresponding to task $k$ and time $t$ is denoted as $\{\bs^{(k)}_{t,i},a^{(k)}_{t,i},\bs'^{(k)}_{t,i}\}_{i=1}^{n_k}$, where $\bs'^{(k)}_{t,i}=\bs^{(k)}_{t+1,i}$ represents the next state. These data points are independently and identically distributed (i.i.d.) across trajectories, indexed by $i$.
In the population version, the ground-truth transition density $\rho_t^{(k)}$ minimizes the following quantity:
\begin{align*}
J(g)&:=\frac{1}{2}\int \int \int (g(\bs,a,\bs')-\rho_t^{(k)}(\bs,a,\bs'))^2 p_t^{(k)}(\bs,a)d\bs da d\bs' \\
&=\frac{1}{2}\int \int \int g(\bs,a,\bs')^2 p_t^{(k)}(\bs,a)d\bs da d\bs'-\int \int \int g(\bs,a,\bs')p_t^{(k)}(\bs,a,\bs')d\bs da d\bs' \\&\qquad+\frac{1}{2}\int \int \int \rho_t^{(k)}(\bs,a,\bs')^2 p_t^{(k)}(\bs,a)d\bs da d\bs',
\end{align*}
where $p_t^{(k)}(\bs,a)$ and $p_t^{(k)}(\bs,a,\bs')$ denote the joint densities of $(\bs^{(k)}_{t,i},a^{(k)}_{t,i})$ and $(\bs^{(k)}_{t,i},a^{(k)}_{t,i},\bs'^{(k)}_{t,i})$ respectively, with respect to the product measure of Lebesgue and counting measures.
The last term is independent of $g$ and can be dropped.  Replacing population densities with empirical versions leads to a square-loss M-estimator for $\rho^{(k)}_t(\bs,a,\bs') := p^{(k)}_t(\bs'|\bs,a)$ using deep ReLU networks:
\begin{equation*}
\underset{g\in \calG(\bar L,\bar M,\bar N,\bar B)}{\arg\min} \frac{1}{2n_k}\int \sum_{i=1}^{n_k} g(\bs^{(k)}_{t,i},a^{(k)}_{t,i},\bs')^2  d\bs'-\frac{1}{n_k}\sum_{i=1}^{n_k} g(\bs^{(k)}{t,i},a^{(k)}_{t,i},\bs'^{(k)}_{t,i}),
\end{equation*}
where the neural network architecture for density ratio estimation differs in size from the one used to estimate the optimal $Q^*$-function.
This estimator involves a high-dimensional integration over $g$, making computation infeasible. We approximate the integral by sampling from $\calS$, leading to:
\begin{equation*}
\label{equ:density-estimator}
\hat\rho_t^{(k)}:=\underset{g\in \calG(\bar L,\bar M,\bar N,\bar B)}{\arg\min} \frac{1}{2n_k}\sum_{i=1}^{n_k} g(\bs^{(k)}_{t,i},a^{(k)}_{t,i},\bs^\circ_i)^2-\frac{1}{n_k}\sum_{i=1}^{n_k} g(\bs^{(k)}_{t,i},a^{(k)}_{t,i},\bs'^{(k)}_{t,i})
\end{equation*}
where $\{\bs^\circ_i\}_{i=1}^{n_k}$ are i.i.d.~uniform samples from $\calS$.
In practice, we can refine the estimator with a projection and normalization step:
\begin{equation*}
\tilde\rho_t^{(k)}=c_N(\bs,a) \max\{\hat\rho_t^{(k)},0\}, 
\end{equation*} 
where $c_N(\bs,a)$ ensures $\int\tilde\rho_t^{(k)}d\bs'=1$ for every $(\bs,a)$, stabilizing the estimator without inflating estimation error.

\smallskip \noindent \textbf{Estimating transition density ratio without density transfer.} 
Algorithm \ref{alg:trans-ratio} details the process of estimating the transition ratio using deep neural networks under conditions of total dissimilarity. To enhance stability, we incorporate a truncation step in forming the density ratio estimator $\omega$, as the density appears in its denominator. For simplicity, we assume $\Upsilon_1$ (or a lower bound thereof) is known. The final estimator is defined as:
\begin{equation*}
    \hat\omega_{t}^{(k)}:=\frac{\hat\rho_{t}^{(0)}}{\max\{\hat\rho_{t}^{(k)},\Upsilon_1\}}. 
\end{equation*}

\begin{algorithm}[tb]
\SetKwInOut{Input}{Input}
\SetKwInOut{Output}{Output}
\Input{Transition tuples $\{\{\bs_i, a_i,\bs'_i\}_{i\in[n]}$.}

\tcc{Calculate $Q^*$-function by backward induction.}

Solve
\begin{equation}
\hat\rho:=\underset{g\in \calG(\bar L,\bar M,\bar N,\bar B)}{\arg\min} \frac{1}{2n}\sum_{i=1}^{n} g(\bs_i,a_i,\bs^\circ_i)^2-\frac{1}{n}\sum_{i=1}^{n} g(\bs_i,a_i,\bs'_i), 
\end{equation}
where $\{\bs^\circ_i\}_{i=1}^{n}$ are i.i.d.~uniformly generated over $\calS$.

\Output{$\hat\rho(\bs,a,\bs')$.}

\caption{Deep NN Estimation of the Conditional Transition Density}
\label{alg:marginal-density-dnn}
\end{algorithm}

\begin{algorithm}[ht!]
\SetKwInOut{Input}{Input}
\SetKwInOut{Output}{Output}
\Input{Target transition tuples $\{\{\bs^{(0)}_{t,i}, a^{(0)}_{t,i},\bs^{(0)}_{t+1,i}\}_{t\in [T]}\}_{i\in[n_0]}$, \\
source transition tuples $\{\{\{\bs^{(k)}_{t,i},a^{(k)}_{t,i},\bs^{(k)}_{t+1,i}\}_{t\in[T]}\}_{i\in[n_k]}\}_{k=1}^K$.}


\tcc{Calculate transition ratio for each task and each step.}

\For{$k\in\{0\}\cup[K]$}{
\For{$t\in[T]$}{ 
Call {\tt Algorithm \ref{alg:marginal-density-dnn}} with inputs $\{\bs^{(k)}_{t,i},a^{(k)}_{t,i},\bs^{(k)}_{t+1,i}\}_{i\in[n_k]}$ to obtain $\hat\rho_{t}^{(k)}(\bs,a,\bs')$ for $k\in\{0\}\cup[K]$. 

Set $\hat\omega_{t}^{(k)}:={\hat\rho_{t}^{(0)}}/{\max\{\hat\rho_{t}^{(k)},\Upsilon_1\}}$.
}
}

\Output{$\{\hat\omega_{t}^{(k)}(\bs'|\bs,a)\}$ for $t\in[T]$ and $k\in[K]$.}

\caption{Deep NN Estimation of Transition Ratios without Density Transfer.}
\label{alg:trans-ratio}
\end{algorithm}

\subsection{Transition Density Ratio Estimation with Transfer} \label{sec:transition-ratio-estimation-density-transfer}

When the target dataset contains sufficient samples, $p^{(0)}_t(\bs'\cond\bs,a)$ can be estimated with adequate accuracy. However, in typical transfer learning scenarios where few target samples are available, estimation becomes challenging. A natural approach is to assume similarity between the conditional densities of the source and target, enabling ``density transfer.'' This assumption is particularly relevant in economic or medical settings, where transitions across different tasks are often driven by common factors and thus exhibit similarities. To formalize this idea, we assume that the ratio of conditional densities possesses high dimension-adjusted degree of smoothness. For simplicity, we consider similarity with only one fixed source $k$, though this can be extended to multiple sources using a similar approach.

\begin{assumption}[Transition Boundedness and Transferability]
\label{asp:density-boundedness-transferability}
Let $\rho^{(k)}_t(\bs,a,\bs') := p^{(k)}_t(\bs'|\bs,a)$ for $k\in\{0\}\cup[K]$. We assume that 
\begin{enumerate}
\item [(i)] (Boundedness.) $\Upsilon_1\le|\rho_t^{(k)}|\le \Upsilon_2$.
\item [(ii)] (Smoothness.) $\rho_t^{(k)}\in \calH_3$ with $\gamma(\calH_3)=\gamma_3$.
\item [(iii)] (Transferability.)
${\rho_t^{(0)}} / {\rho_t^{(k)}} \in \mathcal{H}_4$ with $\gamma(\mathcal{H}_4) = \gamma_4$, and $\gamma_3 \leq \gamma_4$.
\end{enumerate}
\end{assumption}

We propose a two-step transfer algorithm for estimating the transition ratio with transfer, as detailed in Algorithm \ref{alg:trans-ratio-transfer}. The approach can be summarized as follows: First, we estimate the transition density of task $k$ (the source task), as we have usually more source data. Then, we use the target task data to debias this transition density estimate. This debiasing step effectively learns the ratio of the target and source transition densities, given the well-estimated source transition. By structuring the algorithm this way, we directly learn the ratio of the two transition densities as follows:
\begin{align*}
\hat\omega_t^{(k), {\rm tr}}:&=\arg\min_{g\in \calG(\bar L_2,\bar N_2,\bar M_2,\bar B_2)} \frac{1}{2n_0}\sum_{i=1}^{n_0} (g\cdot \hat\rho_t^{(k)})^2(\bs^{(0)}_{t,i},a^{(0)}_{t,i},\bs^\circ_i)-\frac{1}{n_0}\sum_{i=1}^{n_0} (g\cdot\hat\rho_t^{(k)})(\bs^{(0)}_{t,i},a^{(0)}_{t,i},\bs'^{(0)}_{t,i})
\end{align*}where $g\cdot \hat\rho_t^{(k)}$ is the point-wise product of $g$ and $\hat\rho_t^{(k)}$ as a function.

\begin{algorithm}[tb]
\SetKwInOut{Input}{Input}
\SetKwInOut{Output}{Output}
\Input{Target transition tuples $\{\{\bs^{(0)}_{t,i}, a^{(0)}_{t,i},\bs^{(0)}_{t+1,i}\}_{t\in [T]}\}_{i\in[n_0]}$, \\
source transition tuples $\{\{\{\bs^{(k)}_{t,i},a^{(k)}_{t,i},\bs^{(k)}_{t+1,i}\}_{t\in[T]}\}_{i\in[n_k]}\}_{k=1}^K$.}


\tcc{Calculate targeting weights.}
Call {\tt function} to calculate $\{\hat\omega_{t}^{(k)}(\bs'|\bs,a)\}$ for $t\in[T]$ and $k\in[K]$. 

\tcc{Calculate $Q^*$-function by backward induction.}

\For{$k\in[K]$}{
\For{$t\in[T]$}{ 
Call {\tt Algorithm \ref{alg:marginal-density-dnn}} with inputs $\{\bs^{(k)}_{t,i},a^{(k)}_{t,i},\bs^{(k)}_{t+1,i}\}_{i\in[n_k]}$ to obtain $\hat\rho_{t}^{(k)}(\bs,a,\bs')$and let $\hat \rho_t^{(k)}:=\max\{\hat \rho_t^{(k)},\Upsilon_1\}$. 

Calculate 
\begin{equation*}
\hat\omega_t^{(k)}:=\underset{g\in \calG(\bar L_2,\bar N_2,\bar M_2,\bar B_2)}{\arg\min} \frac{1}{2n_0}\sum_{i=1}^{n_0} (g\cdot \hat\rho_t^{(k)})^2(\bs^{(0)}_{t,i},a^{(0)}_{t,i},\bs^\circ_i)-\frac{1}{n_0}\sum_{i=1}^{n_0} (g\cdot\hat\rho_t^{(k)})(\bs^{(0)}_{t,i},a^{(0)}_{t,i},\bs'^{(0)}_{t,i})
\end{equation*}
where $g\cdot \hat\rho_t^{(k)}$ is the point-wise product of $g$ and $\hat\rho_t^{(k)}$ as a function.
}
}

\Output{$\{\hat\omega_{t}^{(k)}(\bs'|\bs,a)\}$ for $t\in[T]$ and $k\in[K]$.}

\caption{Transition Ratio Estimation with Density Transfer}
\label{alg:trans-ratio-transfer}
\end{algorithm}

\section{Theoretical Results with DNN Approximation} 
\label{sec:theory}

We begin by clarifying the random data generation process for the aggregated pool of source and target samples and defining the error terms we aim to bound. For the $k$-th task, $n_k$ trajectories are generated independently and identically distributed (i.i.d.), with $\{(\bs^{(k)}_{t,i},a^{(k)}_{t,i},\bs'^{(k)}_{t,i})\}_{i=1}^{n_k}$ sampled i.i.d. from $\PP_t^{(k)}$. We define the aggregate offline distribution as $\PP_{t}^{\text{agg}}=\sum_{k=0}^K \upsilon_k \PP_t^{(k)}$, where $n_\calM$ aggregated samples are i.i.d. from $\PP^{\operatorname{agg}}_t$.

\subsection{Error Bounds for RWT Transfer Deep $Q$-Learning}
\label{sec:Q*-upper-bound}

For a given estimation error bound of density ratios, 
we now examine the error propagation of our algorithm. 
Let $\PP_{t}^{\operatorname{agg}}$ denote the joint distribution of sample tuples $(\bs_t^{(k)}, a_t^{(k)}, \bs_{t+1}^{(k)})$ from all source and target tasks $k\in[K]$ at stage $t$.
We present some necessary assumptions for theoretical development. 
\begin{assumption}[Positive Action Coverage]
\label{asp:coverage}
There exists a constant $\underline{c}$ such that for every $t$ and almost surely for every $\bs,a$, 
\[
\PP_{t}^{\operatorname{agg}}(A_t=a|S_t=\bs) \ge  \underline{c}. 
\]
\end{assumption}

\begin{assumption}[Bounded Covariate Shift]
\label{asp:covariate-shift}
The Radon-Nikodym derivative of $\PP_t^{\operatorname{agg}}$ and $\PP_t^{(0)}$ satisfies 
\[\eta\le\frac{d\PP_t^{\operatorname{agg}}(\bs,a)}{ d\PP_t^{(0)}(\bs,a)}\le \frac{1}{\eta},\quad a.s. \]
\end{assumption}

\begin{assumption}[Regularity]
\label{asp:regularity}
We assume 
for every $t,\bs,a$, we have $Q^*_t(\bs,a)\le 1$.
\end{assumption}
 
Assumption \ref{asp:coverage} requires the the aggregated behavior policy has lower-bounded minimum propensity score. This is used in converting the optimal $Q^*$-function estimation bound to optimal $V$-function estimation bound. Assumption \ref{asp:covariate-shift} is common in transfer learning literature \citep{ma2023optimally}. 
The constant $1$ in Assumption \ref{asp:regularity} is just a normalization. As a special case, this boundedness assumption holds if the reward is upper bounded by $\max\{\frac{1}{T},1-\gamma\}$.

The following theorem explicitly related the estimation error of $Q^*$ to the estimation error of the density ratio.
The proofs are provided in Appendix B in the supplemental materials.
\begin{theorem} \label{thm:density-to-Q*}
Consider the transfer RL setting with $K+1$ finite-horizon non-stationary MDPs: $\calM^{(k)} = \braces{\calS, \calA, P^{(k)}, r^{(k)}, \gamma, T}$ for $k\in\{0\}\cup[K]$. 
Let $\hat{Q}_t^{\rm tr}$ denote the estimator obtained by Algorithm \ref{alg:rwt-trans-q-general} with DNN approximation \eqref{eqn:Q^p_t}.
Under Assumptions \ref{asp:density-boundedness-transferability} (i), \ref{asp:coverage}, \ref{asp:covariate-shift}, and \ref{asp:regularity}, with probability at least $1-7Te^{-u}$, for every stage $t\in[T]$, we have 
\begin{equation} \label{eqn:Q*-upper-bound}
\begin{aligned}
\|\hat{Q}_t^{\rm tr}-Q^*_t\|^2_{2,\PP_{t}^{(0)}}
& \lesssim 
(T-t)\max\{\kappa,1\}^{T-t}
\left(
\underbrace{\left(\frac{J\log n_0}{n_0}\right)^{\frac{2\gamma_2}{2\gamma_2+1}}}_\text{est.~err.~of { reward} difference}
+ \underbrace{\frac{1}{\eta}\left(\frac{J\log n_\calM}{n_\calM}\right)^{\frac{2\gamma_1}{2\gamma_1+1}}}_\text{est.~err.~of { reward} aggregation} 
\right . \\
& ~~~~~~\left. 
+ \underbrace{\frac{\gamma^2 T^2}{\eta}\max_{t\le\tau\le T}\hat\Omega(\tau)}_\text{est.~err.~of transition ratio} +\frac{u}{\min\{n_0,n_\calM \eta\}}
\right), 
\end{aligned}
\end{equation}
where $J:=\abs{\calA}$, $n_0$ and $n_{\cM}$ are the number of trajectories of the target tasks and the total tasks respectively, $\hat\Omega(t)=\frac{1}{n_\calM}\sum_{i=1}^{n_\calM}|\hat \omega_{t,i}^{(k_i)}-\omega_{t,i}^{(k_i)}|^2$ denotes the estimation error of transition ratio, $\gamma_1$ and $\gamma_2$ defined in Assumption \ref{asp:task-similarity-1} are the complexity of true $Q^*$ functions, and reward differences $\delta^*$'s, $\kappa:=\left(\frac{\gamma^2}{\underline{c}}+\frac{\gamma^2 \Upsilon^2}{\underline{c}\eta^2}\right)$, $\Upsilon=\Upsilon_2/\Upsilon_1$, $\Upsilon_2$, $\Upsilon_1$, $\eta$ and $\bar{c}$ are defined in Assumptions \ref{asp:density-boundedness-transferability} (i), \ref{asp:coverage}, \ref{asp:covariate-shift}, and \ref{asp:regularity}.
\end{theorem}

The error upper bound comprises three additive components: estimation errors from task difference, task aggregation, and transition density ratio. { While the exponential dependence in the coefficient remains unavoidable without additional assumptions, addressing this horizon dependency remains an open challenge for future research. }
Though this bound holds for any density ratio estimator, we will further investigate $\hat\Omega(t)$ in Theorem \ref{thm:density-to-Q*} under three different settings defined in Section \ref{sec:similarity}. 
The first setting assumes {\it total similarity} with $\omega_t^*(k) = 1$ for all $k\in[K]$.
The result can be directly derived from Theorem \ref{thm:density-to-Q*} with $\max_{t\le\tau\le T}\hat\Omega(\tau) = 0$.
Theoretical results under the other two settings of transferable and non-transferable transition densities will be provided in Corollary \ref{thm:density-to-Q*-transition-diff} and \ref{thm:density-to-Q*-transition-trans}, respectively, after we establish the estimation error of transition ration in the next section. 

\begin{remark}[Technique distinctions.] 
While \cite{fan2020theoretical} explores error propagation in deep RL with continuous state spaces, our analysis advances the field in several fundamental ways. First, we explicitly model temporal dependence for real-world relevance, contrasting with \cite{fan2020theoretical}'s experience replay approach. Rather than using sample splitting -- a common but statistically inefficient method in offline RL -- we employ empirical process techniques to handle statistical dependencies.
We also broaden the theoretical scope by removing the function class completeness assumption used in \cite{fan2020theoretical}, though this introduces additional analytical complexity. 
Second, our transfer learning context requires careful consideration of transition density estimation errors, an aspect not present in previous work. 
Finally, we implement backward inductive $Q$-learning for non-stationary MDPs, departing from \cite{fan2020theoretical}'s fixed-point $Q^*$ iteration. Where fixed-point iteration introduces a $1/(1-\gamma)$ term from solving fixed-point equations, our approach estimates $Q^*_t$ at each backward step $t$ using all available data in a batch setting. 
\end{remark}

\begin{remark}[Advantage of transfer RL under total transition similarity]
Consider the setting of {\it total transition similarity}, where $\omega_t^*(k) = 1$ for all $k\in[K]$. The bound consists of two major terms exhibiting classic nonparametric rates, represented by the leading terms in the expression, in addition to terms containing tail probability. These rates are associated with sample sizes $n_0$ for $\gamma_2$ and $n_{\calM}$ for $\gamma_1$. 
In contrast, the convergence rate for backward inductive $Q$-learning without transfer is $\left(\frac{J\log n_0}{n_0}\right)^{\frac{2\gamma_1}{2\gamma_1+1}}$. 
When $\eta>0$ is constant, the advantage of transfer is demonstrated when
$\left(\frac{J\log n_0}{n_0}\right)^{\frac{2\gamma_2}{2\gamma_2+1}}+\frac{1}{\eta}\left(\frac{J\log n_\calM}{n_\calM}\right)^{\frac{2\gamma_1}{2\gamma_1+1}}\ll \left(\frac{J\log n_0}{n_0}\right)^{\frac{2\gamma_1}{2\gamma_1+1}}$.
This advantage is apparent when $n_0\lesssim n_\calM$ and $\gamma_2>\gamma_1$. 
\end{remark}

\begin{remark}[Extension to online transfer RL]
This offline analysis naturally extends to online settings via the Explore-Then-Commit (ETC) framework \citep{chen2022transferred}, detailed in Section \ref{sec:etc}. The optimal transfer strategy may adapt to the volume of target data. For example, when target samples are scarce ($n_0\lesssim n_\calM$), 
utilizing the complete source dataset remains optimal. Once target trajectories exceed this threshold, discarding source data in favor of target-only learning becomes more efficient, achieving an estimation error of $(\frac{J\log n_0}{n_0})^{\frac{2\gamma_1}{2\gamma_1+1}}$ due to the target $Q^*$ function's HCM smoothness of $\gamma_1$. While more sophisticated online transfer algorithms are possible, we defer comprehensive analysis of online transfer RL to future work.
\end{remark}

\subsection{Error Bounds for Transition Density Ratio Estimations}\label{sec:error-bounds-transition-ratio-estimation}

Before we proceed to provide estimation error bounds under the settings of transferable and non-transferable transition, we need to establish the estimation errors of transition density ratios here.  These error bounds for transition density ratio estimation are of independent interest to domain adaptation and transfer learning in policy evaluation.

\subsubsection{Transition Ratio Estimation without Density Transfer}

Our analysis demonstrates that Algorithm \ref{alg:trans-ratio}'s nonparametric least squares approach performs effectively when ReLU neural networks can sufficiently approximate $\rho_t^{(k)}$. For our theoretical analysis, we continue to utilize the hierarchical composition model class and maintain Assumption \ref{asp:density-boundedness-transferability}.
Following \cite{nguyen2010estimating}, we assume the transition density ratio is bounded both above and below almost surely. In practical applications, one can filter out states that occur rarely in the target task.
The following theorem establishes theoretical guarantees for estimating transition densities and their ratios in scenarios where transition densities are so different that no transition density transfer is performed.

\begin{theorem}
\label{thm:error-transition-and-ratio}
Consider the setting of non-transferable transitions under Assumption \ref{asp:density-boundedness-transferability} (i)(ii), \ref{asp:coverage}, and \ref{asp:covariate-shift}. 
We estimate the transition densities $\hat{\rho}_{t}^{(k)}$ and their ratios $\hat\omega_t^{(k)}$ without transfering transition densities via the method described in Section \ref{sec:transition-ratio-estimation-total-dissimilarity}. 
Then, with probability at least $1-e^{-u}$, it holds that
\[
\max\left\{\EE^{(k)}[(\rho_{t}^{(k)}-\hat\rho_{t}^{(k)})^2],\;\frac{1}{n_k}\sum_{i=1}^{n_k}(\rho_{t,i}^{(k)}-\hat\rho_{t,i}^{(k)})^2\right\}
\lesssim 
\frac{u}{n_k} + \left(\frac{\log n_k}{n_k}\right)^{\frac{2\gamma_3}{2\gamma_3+1}}.
\]
Further, for $\hat\Omega(t)=\frac{1}{n_\calM}\sum_{i=1}^{n_\calM}|\hat \omega_{t,i}^{(k_i)}-\omega_{t,i}^{(k_i)}|^2$, with probability at least $1-2T(K+1)e^{-u}$, it holds that
\[\max_{t\in [T]}\hat\Omega(t)
\lesssim \left(\frac{K\log(n_\calM)}{n_\calM}\right)^{\frac{2\gamma_3}{2\gamma_3+1}} + \underbrace{\frac{1}{\eta}\left(\frac{\log n_0}{n_0}\right)^{\frac{2\gamma_3}{2\gamma_3+1}}}_{\text{major est. err. for trans. ratio}} + \frac{u}{\min\{n_0,n_\calM/K\}},\]
where $\gamma_3$ and $\eta$ are defined in Assumption \ref{asp:density-boundedness-transferability} and \ref{asp:covariate-shift}, respectively. 
\end{theorem}



\begin{remark}
When $n_\calM$ exceeds $n_0$ (i.e., $n_\calM\gtrsim n_0$), the bound consists primarily of a standard nonparametric term related to $n_0$. Importantly, the bound is independent of $\min n_k$ since $\hat \Omega(t)$ aggregates all source samples -- tasks with smaller $n_k$ values simply contribute proportionally less to the overall sum.
\end{remark}

\subsubsection{Transition Density Ratio Estimation with Transfers}

The following theorem establishes theoretical guarantees for estimating the transition ratio in the context of transferable transitions where transition transfer is performed.

\begin{theorem}
\label{thm:density-transfer-estimation-error-bound}
Consider the setting of transferable transitions under Assumptions \ref{asp:density-boundedness-transferability}, \ref{asp:coverage} and \ref{asp:covariate-shift}. We estimate the transition densities $\hat{\rho}_{t}^{(k),{\rm tr}}$ and their ratios $\hat\omega_t^{(k),{\rm tr}}$ using transition transfers via the method described in Section \ref{sec:transition-ratio-estimation-density-transfer}.
Then, with probability at least $1-e^{-u}$, it holds that
\[
\EE^{(0)}[(\hat\omega_t^{(k),{\rm tr}}-\omega_t^{(k)})^2]
\lesssim
\EE^{(0)}[|\hat\rho_t^{(k),{\rm tr}}-\rho_t^{(k)}|^2] + \left(\frac{\log n_0}{n_0}\right)^{\frac{2\gamma_4}{2\gamma_4+1}}+\frac{u}{n_0}.
\]
Further, for $\hat\Omega^{\rm tr}(t)=\frac{1}{n_\calM}\sum_{i=1}^{n_\calM}|\hat \omega_{t,i}^{(k_i),{\rm tr}}-\omega_{t,i}^{(k_i)}|^2$, with probability at least $1-3T(K+1)e^{-u}$, it holds that
\begin{align*}
\max_{t\in [T]}\;\hat\Omega^{\rm tr}(t)
& \lesssim
\underbrace{\frac{1}{\eta^2}\left(\frac{K\log n_\calM}{n_\calM}\right)^{\frac{2\gamma_3}{2\gamma_3+1}}}_{\text{est.~err.~of aggregated transition ratio}} + \underbrace{\frac{1}{\eta}\left(\frac{\log n_0}{n_0}\right)^{\frac{2\gamma_4}{2\gamma_4+1}}}_{\text{ est.~err.~of transition difference}} \\
& \hspace{18ex} + \frac{u}{\min\{n_0\eta,n_\calM\eta^2/K\}}+\frac{n_0^{\frac{1}{2\gamma_4+1}}K\log n_\calM}{n_\calM}, 
\end{align*}
where $\gamma_3$ and $\gamma_4$ defined in Assumption \ref{asp:density-boundedness-transferability} (ii) are the complexity of transitions and transition ratios, and $\eta$ is defined in Assumption \ref{asp:covariate-shift}. 
\end{theorem}

\begin{remark}
The convergence rate of the transition density ratio comprises two primary terms: one representing the estimation error of the aggregated transition density ratio, and another accounting for the correction of discrepancy between target and aggregated transition ratios using target samples. In the setting of transferable transitions, the density transfer shows clear advantages over the error bounds in Theorem \ref{thm:error-transition-and-ratio} (without transfer), as $n_0\lesssim n_\calM$ and $\gamma_3 < \gamma_4 \le \calO(1)$.
\end{remark}

\subsection{Error Bounds for RWT Transfer $Q$-Learning with Estimated Transition Density Ratios} \label{sec:general-transition-similarity-Q*-bound}

Using the theoretical results from Section \ref{sec:error-bounds-transition-ratio-estimation}, we directly apply Theorem \ref{thm:density-to-Q*} to two distinct settings: transferable and non-transferable transition densities. We begin with the non-transferable setting, where $\omega^{(k)}_t(\bs'|\bs,a) \ne 1$ but remains bounded, as shown in the following corollary.

\begin{corollary}[Non-transferable transition densities] \label{thm:density-to-Q*-transition-diff}
Under the setting of Theorem \ref{thm:density-to-Q*}, let $\hat{Q}_t^{\rm (tr1)}$ denote the estimator obtained by Algorithm \ref{alg:rwt-trans-q-general} with DNN approximation and value transfers, where the transition ratios $\hat{\rho}_t^{(k)}$ and importance weights $\hat{\omega}_t^{(k)}$ are estimated without transition transfers using the method described in Section \ref{sec:transition-ratio-estimation-total-dissimilarity}. 
Then, with probability at least $1-2T(K+1)e^{-u}$, it holds that
\begin{align*}
\|\hat{Q}_t^{\rm (tr1)}-Q^*_t\|^2_{2,\PP_{t}^{(0)}}
&\lesssim (T-t)\max\{\kappa,1\}^{T-t}\left(\underbrace{\left(\frac{J\log n_0}{n_0}\right)^{\frac{2\gamma_2}{2\gamma_2+1}}}_\text{est.~err.~of {reward} differences}
+ \underbrace{\frac{1}{\eta}\left(\frac{J\log n_\calM}{n_\calM}\right)^{\frac{2\gamma_1}{2\gamma_1+1}}}_\text{est.~err.~of {reward} aggregation}\right.\\
& \left. +\underbrace{\frac{\gamma^2T^2}{\eta^2}\left(\frac{\log n_0}{n_0}\right)^{\frac{2\gamma_3}{2\gamma_3+1}}+\frac{\gamma^2T^2}{\eta}\left(\frac{K\log(n_\calM)}{n_\calM}\right)^{\frac{2\gamma_3}{2\gamma_3+1}}}_\text{est.~err.~of transition ratio, no transfer} \right. \\
&  \left. +\max\{\frac{\gamma^2T^2}{\eta},1\}\frac{u}{\min\{n_0,n_\calM/K,n_\calM \eta\}}\right),
\end{align*}
where $\gamma_1$ and $\gamma_2$ defined in Assumption \ref{asp:task-similarity-1} are the complexity of true $Q^*$ functions, $\gamma_3$ and $\gamma_4$ defined in Assumption \ref{asp:density-boundedness-transferability} (ii) are the complexity of transitions and transition ratios, and $\eta$ is defined in Assumption \ref{asp:covariate-shift}. 
\end{corollary}

\begin{remark}[Advantage of transfer RL without transition transfers]\label{remark:advantage-trl-wo-tt}
Excluding terms with tail probability, the bound comprises three major terms showing classic nonparametric rates. These terms correspond to different sample sizes: $n_0$ for $\gamma_2$ and $n_0$ for $\gamma_3$ associated with value transfer and density ratio estimation respectively, and $n_{\calM}$ for $\gamma_1$ related to aggregated value estimation. When $\eta$, $T$, and $\gamma$ are constant, the advantage of transfer over standard $Q$-learning (which has rate $\left(\frac{J\log n_0}{n_0}\right)^{\frac{2\gamma_3}{2\gamma_3+1}}$) is demonstrated by the inequality:
$\left(\frac{J\log n_0}{n_0}\right)^{\frac{2\gamma_2}{2\gamma_2+1}}+\frac{1}{\eta}\left(\frac{J\log n_\calM}{n_\calM}\right)^{\frac{2\gamma_1}{2\gamma_1+1}} + \left(\frac{\log n_0}{n_0}\right)^{\frac{2\gamma_3}{2\gamma_3+1}} \lesssim \left(\frac{J\log n_0}{n_0}\right)^{\frac{2\gamma_1}{2\gamma_1+1}}$.
This advantage emerges when $n_{\calM}\gg n_0$ and $\gamma_3\ge \gamma_1$. While we do not directly assume $\gamma_3\ge \gamma_1$, this condition can be verified through equation \eqref{eqn-yk}, and is supported by empirical evidence showing that transition density is typically smoother than reward functions. See \cite{shi2022statistical} and references therein. 
\end{remark}

The following corollary instantiate Theorem \ref{thm:density-to-Q*} to the setting of transferable transitions with transition transfer, as described in Section \ref{sec:transition-ratio-estimation-density-transfer}.

\begin{corollary}[Transferable transitions] \label{thm:density-to-Q*-transition-trans}
Under the setting of Theorem \ref{thm:density-to-Q*} and assuming Assumption \ref{asp:density-boundedness-transferability} (ii) holds, let $\hat{Q}_t^{\rm (tr2)}$ denote the estimator obtained by Algorithm \ref{alg:rwt-trans-q-general}. This estimator uses DNN approximation \eqref{eqn:Q^p_t} with both reward transfer and transition density transfer, where the transition density ratios $\hat{\rho}_{t}^{(k),{\rm tr}}$ and importance weights $\hat\omega_t^{(k),{\rm tr}}$ are estimated using the method described in Section \ref{sec:transition-ratio-estimation-density-transfer}. Then, with probability at least $1-3T(K+1)e^{-u}$, we have:
\begin{align*}
\|\hat{Q}_t^{\rm (tr2)}-Q^*_t\|^2_{2,\PP_{t}^{(0)}} & \lesssim (T-t)\max\{\kappa,1\}^{T-t}\left( \underbrace{\left(\frac{J\log n_0}{n_0}\right)^{\frac{2\gamma_2}{2\gamma_2+1}}}_\text{est.~err.~of reward differences}
+ \underbrace{\frac{1}{\eta}\left(\frac{J\log n_\calM}{n_\calM}\right)^{\frac{2\gamma_1}{2\gamma_1+1}}}_\text{est.~err.~of reward aggregation} \right. \\
& \left. + \underbrace{\frac{\gamma^2T^2}{\eta^2}\left(\frac{\log n_0}{n_0}\right)^{\frac{2\gamma_4}{2\gamma_4+1}}}_\text{est.~err.~of transition differences} \right. \left. + \underbrace{\frac{\gamma^2T^2}{\eta^3}\left(\frac{K\log n_\calM}{n_\calM}\right)^{\frac{2\gamma_3}{2\gamma_3+1}}}_\text{est.~err.~of transition aggregation} \right. \\
& \left.+ \max\left\{\frac{\gamma^2T^2}{\eta},1\right\}\frac{u}{\min\{n_0\eta,n_\calM \eta^2/K\}}+\frac{\gamma^2T^2}{\eta}\frac{n_0^{\frac{1}{2\gamma_4+1}}K\log n_\calM}{n_\calM}\right).
\end{align*}

\end{corollary}

\begin{remark}[Transferable transitions: Advantage of transfer RL with transition transfers]
Excluding terms with tail probability, the bound contains four major terms exhibiting classic nonparametric rates, represented by the leading terms in the expression. When $\eta$, $T$, and $\gamma$ are constant, the advantage of transfer over standard $Q$-learning (which has rate $\left(\frac{J\log n_0}{n_0}\right)^{\frac{2\gamma_1}{2\gamma_1+1}}$) is demonstrated by the inequality:
$\left(\frac{J\log n_0}{n_0}\right)^{\frac{2\gamma_2}{2\gamma_2+1}}+\frac{1}{\eta}\left(\frac{J\log n_\calM}{n_\calM}\right)^{\frac{2\gamma_1}{2\gamma_1+1}} + \left(\frac{\log n_0}{n_0}\right)^{\frac{2\gamma_4}{2\gamma_4+1}} + \left(\frac{K\log n_\calM}{n_\calM}\right)^{\frac{2\gamma_3}{2\gamma_3+1}} \lesssim \left(\frac{J\log n_0}{n_0}\right)^{\frac{2\gamma_1}{2\gamma_1+1}}$. 
This advantage emerges when $n_{\calM}\gg n_0$, $\gamma_2\ge \gamma_1$, and $\gamma_4\ge \gamma_1$, where the last condition is ensured by the relationship of $\gamma_3\ge \gamma_1$ as explained in Remark \ref{remark:advantage-trl-wo-tt} and transition transferability ($\gamma_4 \ge \gamma_3$).
\end{remark}



\section{Empirical Studies}
\label{sec:empirical}

\subsection{On-Policy Evaluation for RWT Transfer $Q$ Learning} \label{sec:etc} 

To evaluate our RWT transfer $Q$-learning algorithm, we assess how well the greedy policy derived from $\hat{Q}^{\rm (tr)}$ performs in the target environment through on-policy evaluation. Our experimental approach consists of three distinct phases (Figure \ref{fig:offline-online}):

First, in the target data collection phase, we gather initial RL trajectories from the target environment using a uniform random policy. The duration of this exploration determines the amount of target data available for the transfer learning phase.

Next, during the RWT Transfer $Q$ learning phase, we apply our method to both the collected target data and existing offline source data to compute $\hat{Q}^{\rm (tr)}$.

In the final on-policy evaluation phase, we deploy a greedy policy based on $\hat{Q}^{\rm (tr)}$ and measure its performance in the target environment. For comparison, we also evaluate a baseline approach that uses backward inductive $Q$ learning with only target data to estimate $\hat{Q}^{\rm (sg)}$. We assess the quality of both estimated $Q^*$ functions by measuring the total rewards accumulated when following their respective greedy policies.
\begin{figure}[ht!]
\centering
\includegraphics[width=0.98\linewidth]{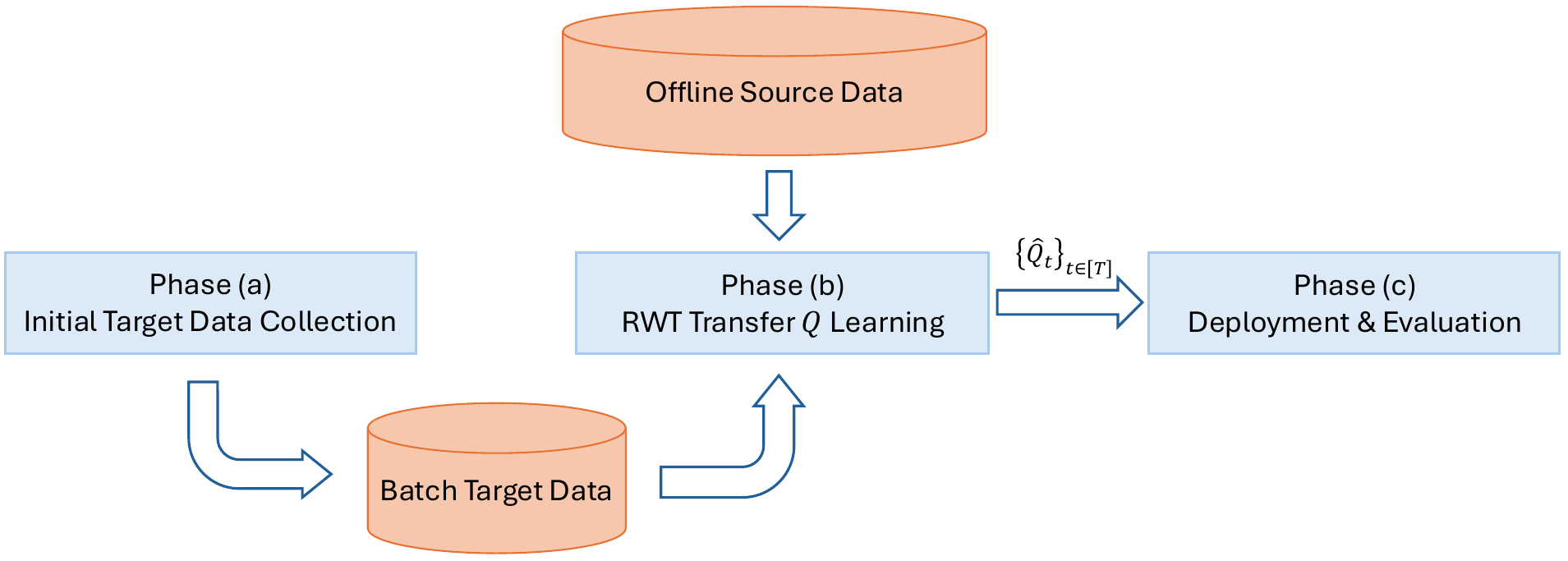}
\caption{\label{fig:offline-online}
Experimental workflow: Phase (a) collects initial target data using uniform random policies. Phase (b) applies RWT Transfer $Q$-learning using both target and source data. Phase (c) conducts on-policy evaluation of the derived greedy policy from $\hat{Q}^{\rm (tr)}$ in the target environment. 
} 
\end{figure}

Our experiments span two environments: a synthetic two-stage Markov Decision Process (MDP) and a calibrated sepsis management simulation using real data. Results show that RWT transfer learning achieves significantly higher accumulated rewards and lower regret compared to learning without transfers, demonstrating robust performance across these distinct settings.

\subsection{Two-Stage MDP with Analytical Optimal $Q^*$ Function}   \label{sec:chakraborty}
\noindent\textbf{Data Generating MDP.} 
The first environment in which we evaluate our method is a two-stage MDP ($T=2$) with binary states $\calX = \braces{-1,1}$ and actions $\calA=\braces{-1,1}$, adapted from \cite{chakraborty2010inference} and \cite{song2015penalized}. This simple environment provides an analytical form of the optimal $Q^*$ function, enabling explicit comparison of regrets during online learning. The states $X_t$ and actions $A_t$ are generated as follows.
At the initial stage ($t=1$), the states and actions are randomly generated, and in the next and final stage (t=2), the state depends on the outcomes of the state and action at the initial stage and is generated according to a logistic regression model.  Explicitly, 
\begin{equation*}
\begin{aligned}
&\Pr\paren{X_1=-1} = \Pr\paren{X_1=1} = 0.5, \\
&\Pr\paren{A_t=-1} = \Pr\paren{A_t=1} = 0.5, \quad t = 1, 2, \\
&\Pr\paren{X_2=1 | X_1, A_1} = 1 - \Pr\paren{X_2 =-1| X_1, A_1} = {\rm expit}\paren{b_1 X_1 + b_2 A_1},
\end{aligned}
\end{equation*}
where ${\rm expit}\paren{x} = \exp\paren{x} / \paren{1 + \exp\paren{x}}$.
The immediate rewards are $R_1 = 0$ and 
\begin{equation*}
R_2 = \kappa_1 + \kappa_2 X_1 + \kappa_3 A_1 + \kappa_4 X_1 A_1 + \kappa_5 A_2 + \kappa_6 X_2 A_2 + \kappa_7 A_1 A_2 + \eps_2,
\end{equation*}
where $\eps_2 \sim \calN\paren{0, 1}$.
Under this setting, the true $Q^*_t$ functions for stage $t=1, 2$ can be analytically derived and are given by
\begin{equation}  \label{eqn:true-q}
\begin{aligned}
Q_2^*\paren{S_2, A_2; \btheta_2} & = \theta_{2,1} + \theta_{2,2} X_1 + \theta_{2,3} A_1 + \theta_{2,4} X_1 A_1 \\
& ~~~ + \theta_{2,5} A_2 + \theta_{2,6} X_2 A_2 + \theta_{2,7} A_1 A_2 \\
Q_1^*\paren{S_1, A_1; \btheta_1} & = \theta_{1,1} + \theta_{1,2} X_1 + \theta_{1,3} A_1 + \theta_{1,4} X_1 A_1,
\end{aligned}
\end{equation}
where the true coefficients $\btheta_t$ are explicitly functions of $b_1$, $b_2$, $\kappa_1,\dots, \kappa_7$ given in equation (H.1) in Appendix H in the supplemental material.
We add more complexity to this MDP by setting the observed covariate $\bs_t \in \RR^p$, $p=31$, consisting of $1$, $S_t$ and remaining elements that are randomly sampled from standard normal. 

\noindent\textbf{Source and Target Environments.} 
We examine transfer learning between two similar MDPs derived from the above model. The MDPs differ in their coefficients $\kappa$'s and consequently $\theta$'s in \eqref{eqn:true-q}. For the target MDP, we set $b_1=1$, $b_2=1$, and $\theta_{2,j} = 1$ for $1 \le j \le 7$, while the source MDP differs only in $\theta_{2,2}^{(1)} = 1.2$. According to equation (H.1) in Appendix H, this leads to stage-one $Q^*$ coefficients of $\theta_{1,1}, \theta_{1,2}, \theta_{1,3}, \theta_{1,4} \approx 2.69, 1.19, 1.69, 1.19$ for the target MDP and $\theta_{1,1}^{(1)}, \theta_{1,2}^{(1)}, \theta_{1,3}^{(1)}, \theta_{1,4}^{(1)} \approx 2.69, 1.39, 1.69, 1.19$ for the source MDP. Thus, the MDPs differ only in $\theta_{1,2}$ for $Q_1$ and $\theta_{2,2}$ for $Q_2$ functions.

\noindent\textbf{The Neural Network Model for $Q$- and $\delta$-functions.} 
Our $Q$-function and difference function implementations utilize a neural network that integrates state-action encoding with a multi-layer perceptron (MLP) architecture:
\begin{align*}
&\text{embedding: } \bx_{\rm enc} = {\rm vec}({\rm concatenate}(\bs, a)\otimes\bM_{\rm ENC}), \nonumber\\
&\bh_1 = {\rm MLP}({\rm DCN}({\rm DCN}(\bx_{\rm enc})), {\rm ReLU}),  \nonumber \\
&y = {\rm MLP}(\bh_1,\; {\rm Linear}).  \nonumber
\end{align*}
During RWT Transfer Q-learning, the output $y$ represents either the $Q$-function value or the difference $\delta$ function value.
The network first encodes inputs using a trainable encoding matrix $\bM_{\rm ENC} \in \RR^{8\times 1}$. The resulting encodings generate a 256-dimensional feature vector that serves as input to the multi-layer perceptron. This MLP processes the 256-dimensional input through a 256-unit hidden layer with ReLU activation functions. The output layer produces a single scalar value without activation, which is suitable for our regression task.
We incorporate Deep \& Cross Network (DCN) blocks, as introduced by \cite{wang2021dcn}, to effectively model high-order interactions between input features while maintaining robustness to noise. These blocks are applied twice in succession to the encoded input before feeding into the MLP layers.

\noindent\textbf{On-Policy Evaluation and Comparison of $\hat{Q}^{\text{(tr)}}$ and $\hat{Q}^{\text{(sg)}}$.}
We generate 10,000 independent trajectories $\paren{\bs_{1,i}, a_{1,i}, r_{1,i}, \bs_{2,i}, a_{2,i}, r_{2,i}}$  from the source MDP and $n_0 \in \{100, 200, \cdots, 1000\}$ trajectories from the target MDP. 

To assess the performance of both Q-function estimates ($\hat{Q}^{\text{(tr)}}$ with transfer learning and $\hat{Q}^{\text{(sg)}}$ without transfer), we deployed their respective greedy policies in the target environment. The evaluation consisted of 100 policy executions for each target dataset size. We measured performance using cumulative rewards over each interaction sequence, adopting an undiscounted reward setting ($\gamma = 1$).

Figure \ref{fig:chakra-online} displays performance metrics averaged over 100 trajectories, comparing various target batch sizes from the exploration phase. We plot both cumulative regret (computed using the analytically-derived optimal Q-function $Q^*$ for this MDP) and cumulative rewards. The analysis reveals that greedy policies derived from the transfer-learned Q-function ($\hat{Q}^{\rm (tr)}$) significantly outperform those from the Q-function ($\hat{Q}^{\rm (sg)}$) without transfer, achieving both lower cumulative regret and higher cumulative rewards.   It demonstrates clearly the benefit of transfer learning in RL.

\begin{figure}[t]
\centering
\includegraphics[width=0.48\linewidth,height=5cm]{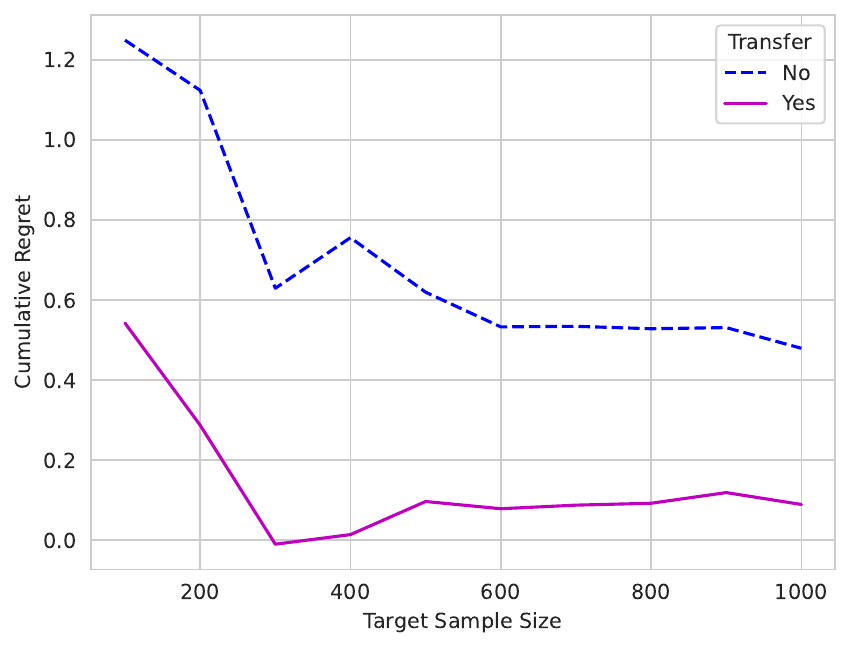}
\includegraphics[width=0.48\linewidth,height=5cm]{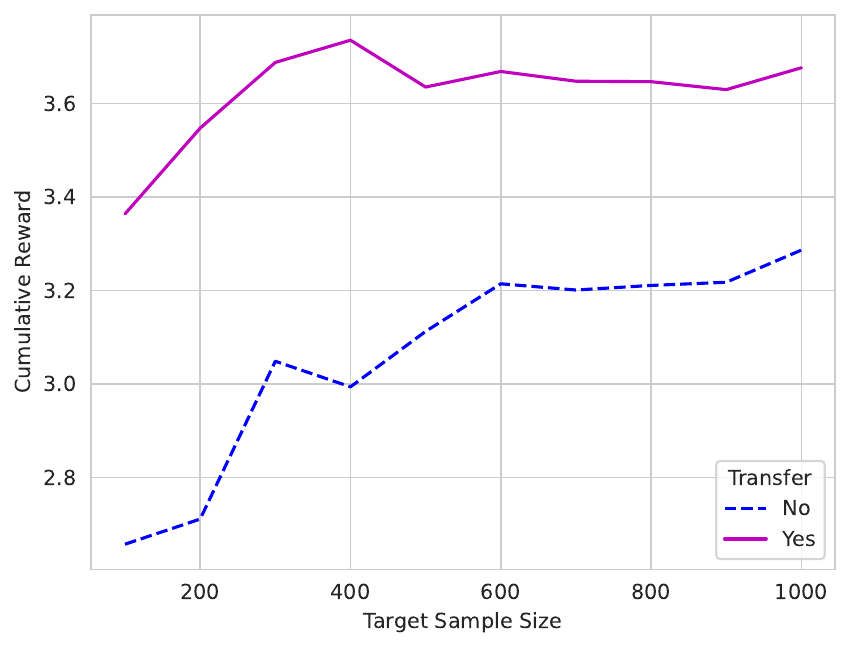}
\caption{\label{fig:chakra-online}
Cumulative regrets (left) and rewards (right) of the online evaluation phase with or without transfer, following the scheme illustrated in Figure \ref{fig:offline-online}. 
The offline source data set has $10,000$ trajectories.  The cumulative regrets and rewards
is shown as a function of the ``Target Sample Size", corresponding to
the amount of target data collected in Phase (a).
The online evaluation phase deploys the greedy policy for both with or without transfer.  
} 
\end{figure}

\subsection{Health Data Application: {\sc Mimic-iii} Sepsis Management} \label{sec:appl-mimic}
\noindent\textbf{Dynamic Treatment Data.}
We evaluated our RWT transfer $Q$-learning method using the MIMIC-III Database (Medical Information Mart for Intensive Care version III, \cite{johnson2016mimic}). This database contains anonymized critical care records collected between 2001-2012 from six ICUs at a Boston teaching hospital. 
For each patient, we encoded state variables as three-dimensional covariates $\bs_{i,t} \in\RR^{3}$ across $T=5$ time steps. The action space captured two key treatment decisions: the total volume of intravenous (IV) fluids and the maximum dose of vasopressors (VASO) \citep{komorowski2018artificial}. The combination of these two treatments yielded $M = 3 \times 3 = 9$ possible actions.
We constructed the reward signal $r_{i,t}$ following established approaches in the literature \citep{prasad2017reinforcement,komorowski2018artificial}. For complete details on data preprocessing, we refer readers to Section J of the supplemental material in \cite{chen2022reinforcement}.

\noindent\textbf{The Neural Network Model.} 
We implemented a neural network architecture for all function estimations, including model calibration, $Q$ functions, and difference functions. 
This architecture combines state and action encoding with a multi-layer perceptron:
\begin{align}
&\text{embedding: } \bs_{\rm enc} = {\rm vec}(\bs\otimes\bM_{\rm SE}), \quad \ba_{\rm enc} = {\rm vec}(a\otimes\bM_{\rm AE}), \nonumber\\
&\bh_1 = {\rm MLP}({\rm concatenate}(\bs_{\rm enc}, \ba_{\rm enc}),\; {\rm ReLU}),  \label{eqn:nn-model} \\
&y = {\rm MLP}(\bh_1,\; {\rm Linear}).  \nonumber
\end{align}
During environment calibration, the output $y$ represents either the reward function or the transition probability density.
During RWT Transfer Q-learning, the output $y$ represents either the Q-function value or the difference function value.
Our architecture encodes three-dimensional states using a learnable state encoder matrix $\bM_{\rm SE} \in \RR^{4\times 1}$ and actions using a learnable action encoder matrix $\bM_{\rm AE} \in \RR^{4\times 1}$. These encodings produce a 16-dimensional input vector (12 dimensions from state encoding and 4 from action encoding), which feeds into a multi-layer perceptron. The MLP takes a 16-dimensional input (12 dimensions from state encoding plus 4 from action encoding) and processes it through a hidden layer of size 16 with ReLU activations. The final layer outputs a single value without activation, appropriate for our regression.

\noindent\textbf{Source and target environment calibration.}
Our study analyzed $20,943$ unique adult ICU admissions, comprising $11,704$ ($55.88\%$) female patients (coded as 0) and $9,239$ ($44.1\%$) male patients (coded as 1). In implementing our Transfer $Q$-learning approach, we designated male patients as the target task and female patients as the auxiliary source task.
To facilitate online evaluation, we constructed neural network-calibrated reinforcement learning environments. Using the architecture described in equation \eqref{eqn:nn-model}, we fitted both reward and transition functions. The source environment was calibrated using $11,704$ trajectories from female patients, while the target environment used $9,239$ trajectories from male patients.
Detailed specifications of the real data calibration process are available in Appendix G in the supplemental material.

\noindent\textbf{RWT Transfer $Q$-learning.}
We generated $n_1 = 10,000$ trajectories from the calibrated source environment and collected varying sizes $n_0 \in \{100, 200, \cdots, 500\}$ of initial target data samples using uniformly random actions from the target environment, as shown in Phase (a) of Figure \ref{fig:offline-online}. For each target data size, we applied our RWT Transfer $Q$-learning method to obtain an estimated $\hat{Q}^{\text{(tr)}}$ function, as illustrated in Phase (b) of Figure \ref{fig:offline-online}. As a baseline comparison, we also estimated $\hat{Q}^{\text{(sg)}}$ using vanilla backward inductive $Q$-learning without transfer \citep{murphy2005generalization}, employing the same neural network architecture from model \eqref{eqn:nn-model} with different target data sizes.

\noindent\textbf{On-Policy Evaluation and Comparison of $\hat{Q}^{\text{(tr)}}$ and $\hat{Q}^{\text{(sg)}}$. }
We evaluated both estimated $Q$ functions ($\hat{Q}^{\text{(tr)}}$ with transfer and $\hat{Q}^{\text{(sg)}}$ without transfer) by deploying their corresponding greedy policies in the target environment. For each target data size, we executed 1,000 interactions using the greedy policy derived from each $\hat{Q}$. During each interaction, we computed the total accumulated reward using an undiscounted setting ($\gamma = 1$). 

Figure~\ref{fig:mimic3-calib-online} shows the average cumulative rewards across 1,000 trajectories, comparing different batch sizes used in the exploration phase. The greedy policies with transfer $\hat{Q}^{\text{(tr)}}$ performed nearly identically across different target sample sizes, with differences only appearing in the third decimal place. This suggests that even small target sample sizes are sufficient for this application. The results clearly show that greedy policies based on $\hat{Q}^{\text{(tr)}}$ with transfer substantially outperformed those using the non-transfer $\hat{Q}^{\text{(sg)}}$ approach in terms of cumulative rewards.
\begin{figure}[tb]
\centering
\includegraphics[width=0.58\linewidth]{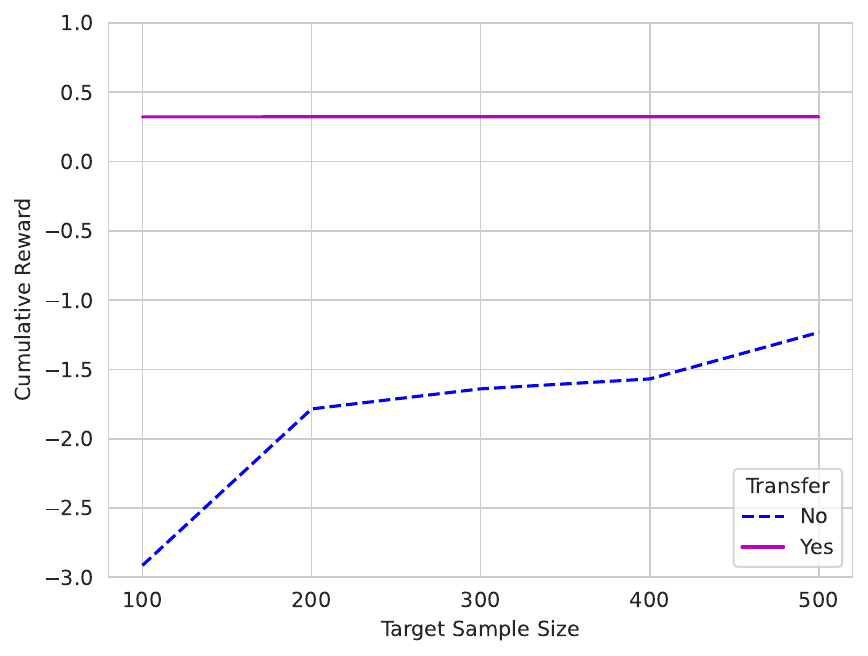}
\caption{\label{fig:mimic3-calib-online}
Cumulative rewards of the online evaluation phase with or without transfer in the MIMIC-3 calibrated environments, averaged over $1,000$ trajectories and following the scheme illustrated in Figure \ref{fig:offline-online}. 
The values of the purple line are nearly identically across different target sample sizes, with differences only appearing in the third decimal place.
The offline source data sets has $10,000$ trajectories. 
The x-axis titled ``Target Sample Size'' represents the number of target data sampled in the phase of initial target data collection. 
The online evaluation deploys the greedy policy for both with or without transfer.} 
\end{figure}

\section{Conclusion} \label{sec:conc}

This paper advances the field of reinforcement learning (RL) by addressing the challenges of transfer learning in non-stationary finite-horizon Markov Decision Processes (MDPs). We have demonstrated that the unique characteristics of RL environments necessitate a fundamental reimagining of transfer learning approaches, introducing the concept of ``transferable RL samples'' and developing the ``re-weighted targeting procedure'' for backward inductive $Q$-learning with neural network function approximation.

Our theoretical analysis provides robust guarantees for transfer learning in non-stationary MDPs, extending insights into deep transfer learning. The introduction of a neural network estimator for transition probability ratios contributes to the broader study of domain shift in deep transfer learning.

This work lays a foundation for more efficient decision-making in complex, real-world scenarios where data is limited but potential impact is substantial. By enabling the leverage of diverse data sources to enhance decision-making for specific target populations, our approach has the potential to significantly improve outcomes in critical societal domains such as healthcare, education, and economics.

While our study has made significant strides, it also opens up new directions for future research, including exploring the applicability of our methods to other RL paradigms and investigating the scalability of our approach to more complex environments.

\spacingset{1.18}
\bibliographystyle{agsm}
\bibliography{main}

\newpage
\setcounter{page}{1}
\begin{appendices}

\begin{center}
{\large\bf SUPPLEMENTARY MATERIAL of \\
``\TITLE''}
\end{center}


This supplementary material is organized as follows. Appendix~\ref{append:notations} provides the lists of notations.
Appendix~\ref{append:q-error} presents the proof of $Q^*$ error bounds with DNN approximation in Section 4.1. Appendix~\ref{append:transition-ratio-estimation} covers the
proofs of transition ratio estimation without density transfer in Section 4.2.1, while Appendix~\ref{append:transition-ratio-estimation-density-transfer} contains the
proofs of transition ratio estimation with density transfer in Section 4.2.2. We include in Appendix~\ref{append:RKHS} the instantiation of RKHS approximation in our general framework. Appendix~\ref{append:extensions} discusses the
extensions of our theory. Appendix~\ref{append:mimic3} provides detailed specifications of the real data calibration process.


\section{Notations}
\label{append:notations}
For any vector $\bx=(x_1,\ldots,x_p)^\top$, let $\|\bx\|:= \|\bx\|_2=(\sum_{i=1}^p x_i^2)^{1/2}$ be the $\ell_2$-norm, and let $\|\bx\|_1=\sum_{i=1}^p |x_i|$ be the $\ell_1$-norm. 
Besides, we use the following matrix norms: 
$\ell_2$-norm $\norm{\bX}_2 := \nu_{max}(\bX)$; $(2,1)$-norm $\norm{\bX}_{2,1} := \underset{\norm{\ba}_1=1}{\max} \norm{\bX\ba}_2 =\max_i \norm{\bx_{i}}_2$;
Frobenius norm $\|\bX\|_F = (\sum_{i,j}x_{ij}^2)^{1/2}$;
nuclear norm $\|\bX\|_*=\sum_{i=1}^n {\nu_{i}(\bX)}$.
When $\bX$ is a square matrix, we denote by $\Tr \paren{\bX}$, $\lambda_{max} \paren{\bX}$, and $\lambda_{min} \paren{\bX}$ the trace, maximum and minimum singular value of $\bX$, respectively.
For two matrices of the same dimension, define the inner product $\angles{\bX_1,\bX_2} = \Tr(\bX_1^\top \bX_2)$.

Let $\bz_1,\cdots, \bz_n$ be i.i.d.~copies of $\bz \sim \mathcal{Z}$ from some distribution $\mu$, $\mathcal{H}$ be a real-valued function class defined on $\mathcal{Z}$. Define the $L_n$ norm (or the empirical $L_2$ norm) and population $L_2$ norm for each $h\in \mathcal{H}$ respectively as
\begin{align*}
\|h\|_n = \left(\frac{1}{n} \sum_{i=1}^n h(\bz_i)^2\right)^{1/2} ~~~~\text{and}~~~~\|h\|_{2,\mu} = \left(\mathbb{E}[h(\bz)^2]\right)^{1/2} = \left(\int h(\bz)^2 \mu(d\bz) \right)^{1/2}.
\end{align*}
We write $\|h\|_{2} = \|h\|_{2,\mu}$ for simple notation when the underlying distribution is clear. 

\section{$Q^*$ Error Bounds with DNN Approximation} \label{append:q-error}

For notational simplicity, we define the following $L_2$ errors whose theoretical guarantees are to be established:
\begin{align*}
\calE(t) &:= \|\hat Q_t-Q^*_t\|^2_{2,\PP_{t}^{\operatorname{agg}}} & \text{(Error under $\PP_{t}^{\operatorname{agg}}$)}, \\
\calE_0(t) &:= \|\hat Q_t-Q^*_t\|^2_{2,\PP_{t}^{(0)}} & \text{(Error under $\PP_{t}^{(0)}$)}, \\
\calE^p(t) &:= \|\hat Q_t^p-Q^{*,\operatorname{agg}}_t\|^2_{2,\PP_{t}^{\operatorname{agg}}} & \text{(Piloting error under $\PP_{t}^{\operatorname{agg}}$)}, \\
\calE_0^p(t) &:= \|\hat Q_t^p-Q^{*,\operatorname{agg}}_t\|^2_{2,\PP_{t}^{(0)}} & \text{(Piloting error under $\PP_{t}^{(0)}$)},
\end{align*}
where $\hat{Q}_t$ is our estimator with RWT transfer for $Q^{*(0)}_t$ for stage $t$, $\hat{Q}^p_t$ is the piloting estimator defined in 
(3.2) which are the pooled backward inductive $Q^*$ estimator for $Q^{*(0)}_t$ for stage $t$, and $Q^{*,\operatorname{agg}}_t$ is defined in 
(2.18). 

We denote sample versions with a ``hat''. For instance:
\begin{align*}
\hat\calE(t) &:= \|\hat Q_t-Q^*_t\|^2_{n_\calM,\hat\PP_{t}^{\operatorname{agg}}} = \frac{1}{n_\calM}\sum_{k=1}^K\sum_{i=1}^{n_k}\left(\hat Q_t(\bs_{t,i}^{(k)},a_{t,i}^{(k)})-Q^*_t(\bs_{t,i}^{(k)},a_{t,i}^{(k)})\right)^2 \\
\hat\calE_0(t) &:= \|\hat Q_t-Q^*_t\|^2_{n_0,\hat\PP_{t}^{(0)}} = \frac{1}{n_0}\sum_{i=1}^{n_0}\left(\hat Q_t(\bs_{t,i}^{(0)},a_{t,i}^{(0)})-Q^*_t(\bs_{t,i}^{(0)},a_{t,i}^{(0)})\right)^2. 
\end{align*}

For the error propagation, we aim to bound $\calE(t)$ and $\calE_0(t)$ by $\calE(t+1)$ and $\calE_0(t+1)$, respectively. While our error bounds integrate over actions, we can analyze the estimation error of the optimal $Q$ function for frequently chosen actions separately. 

\begin{lemma}
\label{lemma:calEp-bound}
Recall that $\hat{\Omega}(t)=\frac{1}{n_\calM}\sum_{i=1}^{n_\calM}|\hat \omega_{t,i}^{(k_i)}-\omega_{t,i}^{(k_i)}|^2$ and assume that $|\omega_{t,i}^{(k)}|\le \Upsilon$. With probability at least $1-3e^{-u}$,
\begin{align*}
\max\{\calE^p(t),\hat\calE^p(t)\}\lesssim \left(\frac{J\log n_\calM}{n_\calM}\right)^{\frac{2\gamma_1}{2\gamma_1+1}}+\gamma^2 \hat\Omega(t)+\frac{\gamma^2 \Upsilon^2}{\underline{c}}\calE(t+1)+\frac{u}{n_\calM}.
\end{align*}
\end{lemma}

\begin{proof}
The upper bound on $|\omega_{t,i}^{(k)}|$ is satisfied by letting $\Upsilon=\frac{\Upsilon_2}{\Upsilon_1}$ and Assumption 4(i).

To streamline the presentation, we abuse a bit of notation and denote $y_{t,i}^{(rwt-0)}=y_{t,i}^{(0)},\ \hat y_{t,i}^{(rwt-0)}=\hat y_{t,i}^{(0)}$, and $\omega_{t,i}^{(0)}=\hat\omega_{t,i}^{(0)}=1$.

We first conduct a bias-variance decomposition of error around the aggregated $Q^*$ function. To be more concrete, for the reweighted responses $\ywh$, we have the following decomposition,
\begin{align}
\ywh=Q^{*\operatorname{agg}}_t(\samplet)+b_{t,i}^{(k)}+v_{t,i}^{(k)}
\end{align}
where the $b_{t,i}^{(k)}$ is the bias term and $v_{t,i}^{(k)}$ is the variance term. 

Recall that  $\ywh=r_{t,i}+\gamma\omegaih\cdot \max_{a\in\calA} \hat Q^{(0)}_{t+1}(\bs_{t+1,i}^{(k)}, a)$ is the RWT pseudo response, $\yw=r_{t,i}+\gamma\omegai\cdot \max_{a\in\calA} Q^{*(0)}_{t+1}(\bs_{t+1,i}^{(k)}, a)$ is the RWT true response. 

Since $\EE[\yw|\samplet]=Q^{*(0)}_t(\samplet) + \delta^{*(k)}_t(\samplet)$, we claim that the bias and the variance have the following form
\begin{align*}
b_{t,i}^{(k)} & = \gamma\hat\omega_{t,i}^{(k)}\cdot \max_{a\in\calA} \hat Q^{(0)}_{t+1}(\bs_{t+1,i}^{(k)}, a) - \gamma\omega_{t,i}^{(k)}\cdot \max_{a\in\calA} Q^{*(0)}_{t+1}(\bs_{t+1,i}^{(k)}, a), \\
v_{t,i}^{(k)} & = \yw-Q^{*\operatorname{agg}}_t(\samplet).
\end{align*}
The bias term can be further decomposed into the bias caused by the estimation error of $Q^*_{t+1}$ and by the estimation error of $\omegai$ as follows,
\begin{align*}|b_{t,i}^{(k)}|&=\Big|\gamma\left(\hat\omega_{t,i}^{(k)}-\omega_{t,i}^{(k)}\right)\cdot \max_{a\in \calA}\hat Q_{t+1}(\bs_{t+1,i}^{(k)}, a)+\gamma\omega_{t,i}^{(k)}\cdot \left(\max_{a\in \calA}\hat Q_{t+1}(\bs_{t+1,i}^{(k)}, a)-\max_{a\in \calA}Q^*_{t+1}(\bs_{t+1,i}^{(k)}, a)\right)\Big|\\&\lesssim 
\gamma |\hat \omega_{t,i}^{(k)}-\omega_{t,i}^{(k)}|+\gamma \Upsilon \iota_{t+1,i}^{(k)}
\end{align*}
where we defined 
$$\iota_{t+1,i}^{(k)}=\Big|\max_{a\in \calA}\hat Q_{t+1}(\bs_{t+1,i}^{(k)}, a)-\max_{a\in \calA}Q^*_{t+1}(\bs_{t+1,i}^{(k)}, a)\Big|.$$ 
The inequality follows from $ |\hat Q_{t+1}|\le |\hat Q_{t+1}^{\operatorname{agg}}|+|\hat Q_{t+1}-\hat Q_{t+1}^{\operatorname{agg}}|\le M_1+M_2\lesssim 1$, which can also be guaranteed if we add a truncation step and $|\omega^{(k)}_{t,i}|\le \Upsilon$.

The variance $v_{t,i}^{k}$ contains two terms. 
\begin{align*}
v_{t,i}^{(k)}=\underbrace{\left(\yw - r^{*(k)}_t(\samplet)-\gamma\EE[\max_{a'}Q^{*{(0)}}_{t+1}(\bs_{t+1},a')|\bs_t=\bs,a_t=a]\right)}_{v_{t,i}^{k}(1)}\\+
\underbrace{\left(r^{*(k)}_t(\samplet)-r^{*\operatorname{agg}}_t(\samplet)\right)}_{v_{t,i}^{k}(2)}.\end{align*}
The first term $v_{t,i}^{k}(1)$ comes from the intrinsic variance of value iteration and the second term $v_{t,i}^{k}(2)$ comes from the aggregation process. We now verify that $v_{t,i}^{(k)}$ is indeed the variance, with the conditional mean on $\samplet$ being 0. While it's straightforward from the Bellman Equation and the definition of $\yw$ that $\EE[v_{t,i}^{(k)}(1)|\samplet]=0$, generally we have  $\EE[v_{t,i}^{(k)}(2)|\samplet]\neq0$ for fixed $k$. However, when we pool the data and relabel them from $i=1$ to $n_\calM$, $k$ is random and can be regarded as $k_i$. Further as $n_k$ is itself drawn from a binomial distribution, we know $\{v_{t,i}^{k_i}(2)\}_{i=1}^{n_\calM}$ are i.i.d. From the definition of aggregate function it is straightforward to check that $\EE[v_{t,i}^{k_i}(2)|\samplet]=0$.


After clarifying the decomposition, we can shift to the analysis of nonparametric least squares. We will first state the following two lemmas. The first one characterizes the distance between $L_n$-norm and $L_2$-norm. The second bounded the tail of weighted empirical process.
\begin{lemma}
\label{lemma:L2-Ln-relation}
Let $z_1,\cdots,z_n\in \calZ$ be i.i.d. copies of $\bz$, $\calG$ be a $b$-uniformly-bounded function class satisfying $\log(\calN_\infty(\epsilon,\calG,\bz_1^n))\le v\log\left(\frac{ebn}{\epsilon}\right)$ for some quantity $v$. Then there exists $c_1,c_2,c_3$ such that as long as $t\ge c_1\sqrt{\frac{v\log n}{n}}$, with probability at least $1-c_2 e^{-c_3nt^2}$, we have
\[\Big|\|g\|_n^2-\|g\|_2^2\Big|\le \frac{1}{2}(\|g\|_2^2+t^2),\quad \forall g\in \calG.\]
\end{lemma}

\begin{proof}
The proof consists of a standard symmetrization technique followed by chaining. See, for example, Theorem 14.1 and Proposition 14.25 in \cite{wainwright2019high}, Theorem 19.3 in \cite{gyorfi2002distribution}, or Lemma 3 in \cite{fan2023factor}. Note here we use covering number instead of pseudo dimension to allow application to the class of value functions which consists of maxima over $Q$-functions.
\end{proof}

\begin{lemma}
\label{lemma:tail-empirical-process}
Let $z_1,\cdots,z_n$ be fixed and $\epsilon_1,\cdots,\epsilon_n$ be i.i.d. sub-Gaussian random variables with variance parameter $\sigma$. Let $\tilde G$ be a subset of $b$-uniformly bounded functions and $\tilde g$ be a fixed function. Suppose for some quantity $v$, it holds that $\log(\calN_\infty(\epsilon,\calG,\bz_1^n))\le v\log\left(\frac{ebn}{\epsilon}\right)$. Then with probability at least $1-c_2\log(1/\epsilon)e^{-t}$, we have for some constants $c_1,c_2$, \[\Big|\frac{1}{n}\sum_{i=1}^n\epsilon_i(g(z_i)-\tilde g(z_i))\Big|\le c_1(\|g-\tilde g\|_n+\epsilon)\sqrt{v_n^2+\frac{t}{n}},\quad \forall g\in \calG\]
where $v_n=\sqrt{\frac{v\log n}{n}}$.
\end{lemma}

\begin{proof}
The proof can be completed by a standard chaining technique followed by peeling device. See, for example, Lemma 4 in \cite{fan2023factor}.
\end{proof}

Recall that we define $\calG(L,N,M,B)$ to be the ReLU network with depth $L$, width $N$, truncation level $M$ and the bound on weights $B$. To facilitate presentation, we also define $\tilde \calG(L,N,M,B)=\{Q:Q(\cdot,a)\in \calG(L,N,M,B),\  \forall a\}$. 
By the approximation results for ReLU neural network \citep{fan2023factor}, we have that for $1\le M\lesssim 1,\log B\asymp \log n$,
\[\sup_{Q^*\in \calH,|Q^*|\le 1}\inf_{g\in \tilde\calG(L,N,M,B)}\|g-Q^*\|_\infty\le (NL)^{-2\gamma(\calH)}.\]
Therefore, pick $1\le M_1\lesssim 1,\log B_1\asymp \log n$, there exists a $g_t^p\in \tilde\calG(L_1,N_1,M_1,B_1)$ such that $\|g_t^p-Q^{*\operatorname{agg}}_t\|_\infty\le (N_1L_1)^{-2\gamma_1}$, where we used the Assumption 3 and 7.

From the optimality of our pooled estimator we have that 

\[\sum_{i=1}^{n_\calM}(\hat y_{t,i}^{(rwt-k_i)}-\hat Q_t^p(\sampleki))^2\le \sum_{i=1}^{n_\calM}(\hat y_{t,i}^{(rwt-k_i)}- g_t^p(\sampleki))^2.\]
After some algebra we get that 
\[\|\hat Q_t^p-Q_t^{*\operatorname{agg}}\|^2_{n_\calM,\hat\PP_t^{\operatorname{agg}}}\le \|g_t^p-Q_t^{*\operatorname{agg}}\|^2_{n_\calM,\hat\PP_t^{\operatorname{agg}}}+\frac{2}{n_\calM}\sum_{i=1}^{n_\calM}(\hat y_{t,i}^{rwt-k_i}-Q_{t,i}^{*\operatorname{agg}})\cdot(\hat Q_{t,i}^p-g_{t,i}^p).\]
By triangle inequality we have 
\[\|\hat Q_t^p-g_t^p\|^2_{n_\calM,\hat\PP_t^{\operatorname{agg}}}\le 2\|\hat Q_t^p-Q_t^{*\operatorname{agg}}\|^2_{n_\calM,\hat\PP_t^{\operatorname{agg}}}+2\|g_t^p-Q_t^{*\operatorname{agg}}\|^2_{n_\calM,\hat\PP_t^{\operatorname{agg}}},\]
which combined with the above inequality and the approximation of $g_t^p$ implies that
\[\|\hat Q_t^p-g_t^p\|^2_{n_\calM,\hat\PP_t^{\operatorname{agg}}}\le 4(N_1L_1)^{-4\gamma_1}+\frac{2}{n_\calM}\sum_{i=1}^{n_\calM}(\hat y_{t,i}^{rwt-k_i}-Q_{t,i}^{*\operatorname{agg}})\cdot(\hat Q_{t,i}^p-g_{t,i}^p).\]

Recall the bias-variance decomposition of $\hat y_{t,i}^{rwt-k_i}-Q_{t,i}^{*\operatorname{agg}}=b_{t,i}^{k_i}+v_{t,i}^{k_i}$. For the bias term we directly use Cauchy-Schwartz Inequality and arrives at
\[\Big|\frac{1}{n_\calM}\sum_{i=1}^{n_\calM}b_{t,i}^{k_i}(\hat Q_{t,i}^p-g_{t,i}^p)\Big|\le \|\hat Q_t^p-g_t^p\|_{n_\calM,\hat\PP_t^{\operatorname{agg}}}\sqrt{\frac{1}{n_\calM}\sum_{i=1}^{n_\calM}(b_{t,i}^{k_i})^2}.\]

For the variance term we apply the lemma to bound the tail of empirical process. To begin with, note that by Lemma 7 in \cite{fan2023factor}, the $L_\infty$ covering number of $\calG (L_1,N_1,M_1,B_1)$ can be bounded by $\log \calN_{\infty}(\epsilon,\calG (L_1,N_1,M_1,B_1),[0,1]^d)\le N_1^2L_1^2 \log(\frac{N_1L_1}{\epsilon})$. Also, $\tilde\calG (L_1,N_1,M_1,B_1)=\prod_{a=1}^J \calG_a (L_1,N_1,M_1,B_1)$. Therefore, letting $v=JN_1^2L_1^2\log(N_1L_1)$ and $v_n=\sqrt{\frac{v\log(n)}{n}}$, it holds that \[\log \calN_{\infty}(\epsilon,\tilde\calG (L_1,N_1,M_1,B_1),\bz_1^n)\le J\log \calN_{\infty}(\epsilon,\calG (L_1,N_1,M_1,B_1),\bz_1^n)\lesssim v\log(\frac{ebn}{\epsilon}).\]

By Lemma \ref{lemma:tail-empirical-process}, with probability at least $1-e^{-u}$,
\[\Big|\frac{1}{n_\calM}\sum_{i=1}^{n_\calM}v_{t,i}^{k_i}(\hat Q_{t,i}^p-g_{t,i}^p)\Big|\lesssim (\|\hat Q_t^p-g_t^p\|_{n_\calM,\hat\PP_t^{\operatorname{agg}}}+v_{n_\calM})\sqrt{v_{n_\calM}^2+\frac{u}{n_\calM}} .\]

Putting the pieces together, we get that 
\begin{align*}
\|\hat Q_t^p-g_t^p\|^2_{n_\calM,\hat\PP_t^{\operatorname{agg}}}\lesssim &(N_1L_1)^{-4\gamma_1}+\|\hat Q_t^p-g_t^p\|_{n_\calM,\hat\PP_t^{\operatorname{agg}}}\sqrt{\frac{1}{n_\calM}\sum_{i=1}^{n_\calM}\left(\gamma |\hat \omega_{t,i}^{k_i}-\omega_{t,i}^{k_i}|+\gamma \Upsilon \iota_{t+1,i}^{(k_i)}\right)^2}\\
+&(\|\hat Q_t^p-g_t^p\|_{n_\calM,\hat\PP_t^{\operatorname{agg}}}+v_{n_\calM})\sqrt{v_{n_\calM}^2+\frac{u}{n_\calM}}.
\end{align*}

Simplifying the terms we obtain that with probability at least $1-e^{-u}$,
\begin{equation}
\label{equ:basic-inequality-simplified-2}\|\hat Q_t^p-g_t^p\|^2_{n_\calM,\hat\PP_t^{\operatorname{agg}}}\lesssim (N_1L_1)^{-4\gamma_1}+\frac{\gamma^2}{n_\calM}\sum_{i=1}^{n_\calM}|\hat \omega_{t,i}^{k_i}-\omega_{t,i}^{k_i}|^2+\frac{\gamma^2 \Upsilon^2}{n_\calM}\sum_{i=1}^{n_\calM}(\iota _{t+1,i}^{(k_i)})^2+v_{n_\calM}^2+\frac{u}{n_\calM}.\end{equation}

Next, we tackle the two bias terms.
Recall the estimation error bound of transition density ratio $\frac{1}{n_\calM}\sum_{i=1}^{n_\calM}|\hat \omega_{t,i}^{k_i}-\omega_{t,i}^{k_i}|^2=\hat \Omega(t)$.

The second bias term is a little bit tricky because it's a sample-version 2-norm of V-function. It turns out we can translate it into population norm and use the Assumption 
5 to translate to the estimation error of $t+1$.

To be more concrete, we need to bound the metric entropy of $\calG_v=\{\max_a g(\cdot,a):g\in \tilde\calG (L_1,N_1,M_1,B_1)\}$.

Let $g_1,\cdots,g_\Pi$ be the $\epsilon$-covering set of $\calG=\calG (L_1,N_1,M_1,B_1)$, then we have that for all $ g\in\calG$, there exists a $\pi(g)$ such that $\|g-g_{\pi(g)}\|_\infty\le \epsilon$. We argue that $\max_{a=1}^J g_{i_a}(\cdot)$ indexed by a $J$-tuple $(i_1,\cdots,i_J)\in [\Pi]^J$ is an $\epsilon$-covering set of $\calG_v$.

In fact, for any function $g_v\in\calG_v$, by definition we have $g_v=\max_a g(\cdot,a)$. Again we can find $i_1,i_2,\cdots,i_J$ such that $\|g(\cdot,a)-g_{i_a}(\cdot)\|_\infty\le \epsilon$ for every $a\in [J]$. It then follows that $\|g_v-\max_{a=1}^J g_{i_a}(\cdot)\|_\infty\le \epsilon$.

The above reasoning shows that $\log N_\infty(\epsilon,\calG_v)\le J\log N_\infty(\epsilon,\calG)$. 

In view of Lemma \ref{lemma:L2-Ln-relation}, we have that there exists some universal constant $c$, such that with probability at least $1-e^{-u}$,
\[\frac{1}{n_\calM}(\iota_{t+1,i}^{k_i})^2\le \frac{3}{2}\EE[(\iota_{t+1,i}^{k_i})^2]+c(v_{n_\calM}^2+\frac{u}{{n_\calM}}).\]

We further connect the population quantities. By the coverage assumption 5, we have that
\begin{align*}\EE[(\iota_{t+1,i}^{(k_i)})^2]&=\EE_{\bs\sim\PP_{t+1}^{\operatorname{agg}}}\Big|\max_{a\in \calA}\hat Q_{t+1}(\bs, a)-\max_{a\in \calA}Q^*_{t+1}(\bs, a)\Big|^2\\&\le\EE_{\bs\sim\PP_{t+1}^{\operatorname{agg}}}\max_{a\in \calA}\Big|\hat Q_{t+1}(\bs, a)-Q^*_{t+1}(\bs, a)\Big|^2\\
&\le \frac{1}{\underline{c}}\EE_{(\bs,a)\sim\PP_{t+1}^{\operatorname{agg}}}\Big|\hat Q_{t+1}(\bs, a)-Q^*_{t+1}(\bs, a)\Big|^2\\
&=\frac{1}{\underline{c}}\calE(t+1).
\end{align*}

Plugging to (\ref{equ:basic-inequality-simplified-2}) and applying a union bound, we obtain with probability at least $1-2e^{-u}$,
\begin{align*}\|\hat Q_t^p-g_t^p\|^2_{n_\calM,\hat\PP_t^{\operatorname{agg}}}&\lesssim (N_1L_1)^{-4\gamma_1}+\gamma^2 \hat\Omega(t)+\frac{\gamma^2 \Upsilon^2}{\underline{c}}\calE(t+1)+v_{n_\calM}^2+\frac{u}{{n_\calM}}\\
&\lesssim (N_1L_1)^{-4\gamma_1}+\gamma^2 \hat\Omega(t)+\frac{\gamma^2 \Upsilon^2}{\underline{c}}\calE(t+1)+\frac{JN_1^2L_1^2\log(N_1L_1)\log(n_\calM)}{n_\calM}+\frac{u}{{n_\calM}}
\end{align*}
where we plugged in the expression of $v_n$.

Set the neural network parameters such that $N_1L_1\asymp \left(\frac{J}{n_\calM}\right)^{\frac{1}{4\gamma_1+2}}$, also as $\|\hat Q_t^p-Q^{*\operatorname{agg}}_t\|^2_{n_\calM,\hat\PP_t^{\operatorname{agg}}}\le 2\|\hat Q_t^p-g_t^p\|^2_{n_\calM,\hat\PP_t^{\operatorname{agg}}}+2\|Q^{*\operatorname{agg}}_t-g_t^p\|^2_{n_\calM,\hat\PP_t^{\operatorname{agg}}}$, we attain that with probability at least $1-2e^{-u}$,
\[\hat\calE^p(t)=\|\hat Q_t^p-Q^{*\operatorname{agg}}_t\|^2_{n_\calM,\hat\PP_t^{\operatorname{agg}}}\lesssim \left(\frac{J\log n_\calM}{n_\calM}\right)^{\frac{2\gamma_1}{2\gamma_1+1}}+\gamma^2 \hat\Omega(t)+\frac{\gamma^2 \Upsilon^2}{\underline{c}}\calE(t+1)+\frac{u}{n_\calM}.\]

Note that we have calculated the metric entropy of $\tilde\calG$. 
Applying the Lemma \ref{lemma:L2-Ln-relation} again and a union bound we get with probability at least $1-3e^{-u}$,
\begin{align*}
\calE^p(t)&=\|\hat Q_t^p-Q^{*\operatorname{agg}}_t\|^2_{2,\PP_t^{\operatorname{agg}}}\lesssim
\|\hat Q_t^p-g_t^p\|^2_{n_\calM,\hat\PP_t^{\operatorname{agg}}}+v_{n_\calM}^2+\frac{u}{n_\calM}
\\&\lesssim \left(\frac{J\log n_\calM}{n_\calM}\right)^{\frac{2\gamma_1}{2\gamma_1+1}}+\gamma^2 \hat\Omega(t)+\frac{\gamma^2 \Upsilon^2}{\underline{c}}\calE(t+1)+\frac{u}{n_\calM},
\end{align*}
and therefore, $\max\{\calE^p(t),\hat\calE^p(t)\}\lesssim \left(\frac{J\log n_\calM}{n_\calM}\right)^{\frac{2\gamma_1}{2\gamma_1+1}}+\gamma^2 \hat\Omega(t)+\frac{\gamma^2 \Upsilon^2}{\underline{c}}\calE(t+1)+\frac{u}{n_\calM}$.

\end{proof}

\begin{lemma}
\label{lemma:calE0-bound}
With probability at least $1-e^{-4t}$, we have
\[\calE_0(t)\lesssim \paren{\frac{J\log n_0}{n_0}}^{\frac{2\gamma_2}{2\gamma_2+1}}+\frac{\gamma^2}{\underline{c}}\calE_0(t+1)+\frac{u}{n_0}+\calE^p_0(t).\]
\end{lemma}

\begin{proof}
Note that although $n_0$ is random, we can condition on fixed $n_0$. The pipeline of this proof is similar to Lemma \ref{lemma:calEp-bound}, with the bias now coming from pooling at the same time stage. Note that in this proof, all the neural network size is confined to be $(L_2,N_2,M_2,B_2)$ and sometimes we omit it.

Again, using the approximation results for ReLU neural networks \citep{fan2023factor}, there exists a $g_t\in \tilde\calG$ such that $\|Q_t^*-Q_t^{*\operatorname{agg}}-g_t\|_\infty\le (N_2L_2)^{-2\gamma_2}$, where we used the Assumption 3
and 7
.

From the optimality of the debiased estimator, we have that 
\[\sum_{i=1}^{n_0}(\hat y_{t,i}^{(0)}-\hat Q^p_{t,i}-\hat\delta_t(\bs_{t,i}^{(0)},a_{t,i}^{(0)}))^2\le \sum_{i=1}^{n_0}(\hat y_{t,i}^{(0)}-\hat Q^p_{t,i}-g_t(\bs_{t,i}^{(0)},a_{t,i}^{(0)}))^2.\]

After some algebra we have that 
\[\|Q_t^*-Q_t^{*\operatorname{agg}}-\hat\delta_t\|_{n_0,\hat\PP_t^{(0)}}^2\le \|Q_t^*-Q_t^{*\operatorname{agg}}-g_t\|_{n_0,\hat\PP_t^{(0)}}^2+\frac{2}{n_0}\sum_{i=1}^{n_0}(\hat y_{t,i}^{(0)}-Q^*_{t,i}-\hat Q^p_{t,i}+Q^{*\operatorname{agg}}_{t,i})\cdot (\hat\delta_{t,i}-g_{t,i}).\]

By the definition of $g_t$, the first term on the right-hand-side can be bounded by $(N_2L_2)^{-2\gamma_2}$. We again decompose the second term into the bias part and variance part.

$\hat Q^p_{t,i}-Q^{*\operatorname{agg}}_{t,i}$ is viewed as bias, while for $\hat y_{t,i}^{(0)}-Q^*_{t,i}$, we treat the error incurred by estimation error after time $t$ as bias while the randomness at this stage as variance. Specifically, recall the pseudo outcome $\hat y_{t,i}^{(0)}=r_{t,i}^{(0)}+\gamma\max_{a\in\calA}\hat Q_{t+1}(\bs_{t+1,i}^{(0)},a)$.
Define the true outcome $y_{t,i}^{(0)}=r_{t,i}^{(0)}+\gamma\max_{a\in\calA} Q^*_{t+1}(\bs_{t+1,i}^{(0)},a)$, we have that $\EE[y_{t,i}^{(0)}-Q_{t,i}^*|\bs_{t,i}^{(0)},a_{t,i}^{(0)}]=0$ by the Bellman Equation.

On the other hand, we can view $\zeta_{t,i}=\hat y_{t,i}^{(0)}-y_{t,i}^{(0)}$ as another term of bias.

Therefore, the basic inequality boils down to 
\begin{align*}
\|\hat \delta_t-g_t\|^2_{n_0,\hat\PP_t^{(0)}}&\lesssim (N_2L_2)^{-4\gamma_2}+\frac{1}{n_0}\sum_{i=1}^{n_0}(\hat y_{t,i}^{(0)}-y_{t,i}^{(0)}+y_{t,i}^{(0)}-Q^*_{t,i}-\hat Q^p_{t,i}+Q^{*\operatorname{agg}}_{t,i})\cdot (\hat\delta_{t,i}-g_{t,i})\\
&\lesssim (N_2L_2)^{-4\gamma_2}+\underbrace{\Big|\frac{1}{n_0}\sum_{i=1}^{n_0}(y_{t,i}^{(0)}-Q^*_{t,i})\cdot (\hat\delta_{t,i}-g_{t,i})\Big|}_{T_1}+\underbrace{\|\hat \delta_t-g_t\|_{n_0,\hat\PP_t^{(0)}}\sqrt{\frac{1}{n_0}\sum_{i=1}^{n_0}(\hat y_{t,i}^{(0)}-y_{t,i}^{(0)})^2}}_{T_2}\\
&+\underbrace{\|\hat \delta_t-g_t\|_{n_0,\hat\PP_t^{(0)}}\sqrt{\frac{1}{n_0}\sum_{i=1}^{n_0}(\hat Q^p_{t,i}-Q^{*\operatorname{agg}}_{t,i})^2}}_{T_3}.
\end{align*}

We deal with $T_1$--$T_3$ separately. Define $v=JN_2^2L_2^2\log(N_2L_2)$ and $v_n=\sqrt{\frac{v\log n}{n}}$.
By Lemma \ref{lemma:tail-empirical-process} and similar analysis on metric entropy of $\tilde\calG$, we have with probability at least $1-e^{-u}$,
\[T_1\le (\|\hat \delta_t-g_t\|_{n_0,\hat\PP_t^{(0)}}+v_{n_0})\sqrt{v_{n_0}^2+\frac{u}{n_0}}.\]

For $T_2$, we again bridge it through the 
population version. 
Similarly, by bounding the metric entropy of $\calG_v=\{\max_a g(\cdot,a):g\in\tilde\calG\}$, we can apply Lemma \ref{lemma:L2-Ln-relation} and arrive at
\[\frac{1}{n_0}\sum_{i=1}^{n_0}(\hat y_{t,i}^{(0)}-y_{t,i}^{(0)})^2\lesssim \EE[(\hat y_{t,i}^{(0)}-y_{t,i}^{(0)})^2]+v_{n_0}^2+\frac{u}{n_0}\]with probability at least $1-e^{-u}$. We also have, by Assumption 
5, that
\[\EE[(\hat y_{t,i}^{(0)}-y_{t,i}^{(0)})^2]\le \frac{\gamma^2}{\underline{c}}\calE_0(t+1).\]
Therefore, we have $T_2\lesssim \|\hat \delta_t-g_t\|_{n_0,\hat\PP_t^{(0)}}\sqrt{\frac{\gamma^2}{\underline{c}}\calE_0(t+1)+v_{n_0}^2+\frac{u}{n_0}}$.

$T_3$ is just given by $T_3=\|\hat \delta_t-g_t\|_{n_0,\hat\PP_t^{(0)}}\sqrt{\hat\calE_0^p(t)}$. Putting bounds on $T_1$--$T_3$ together we can obtain that with probability at least $1-2e^{-u}$,
\[\|\hat \delta_t-g_t\|_{n_0,\hat\PP_t^{(0)}}^2\le (N_2L_2)^{-4\gamma_2}+v_{n_0}^2+\hat\calE_0^p(t)+\frac{\gamma^2}{\underline{c}}\calE_0(t+1)+\frac{u}{n_0}.\]

Again using Lemma \ref{lemma:L2-Ln-relation}, we have 
with probability at least $1-3e^{-u}$,
\[\|\hat \delta_t-g_t\|_{n_0,\PP_t^{(0)}}^2\le (N_2L_2)^{-4\gamma_2}+v_{n_0}^2+\hat\calE_0^p(t)+\frac{\gamma^2}{\underline{c}}\calE_0(t+1)+\frac{u}{n_0}.\]

It follows that
\begin{align*}\|\hat Q_t-Q_t^*\|^2_{n_0,\PP_t^{(0)}}&\lesssim \|\hat \delta_t-g_t\|_{n_0,\PP_t^{(0)}}^2+\|Q^*_t-Q^{*\operatorname{agg}}_t-g_t\|_{n_0,\PP_t^{(0)}}^2+\|\hat Q^p_t-Q^{*\operatorname{agg}}_t\|_{n_0,\PP_t^{(0)}}^2\\
&\lesssim (N_2L_2)^{-4\gamma_2}+v_{n_0}^2+\hat\calE_0^p(t)+\frac{\gamma^2}{\underline{c}}\calE_0(t+1)+\frac{u}{n_0}+\calE^p_0(t)\\
&\lesssim (N_2L_2)^{-4\gamma_2}+v_{n_0}^2+\frac{\gamma^2}{\underline{c}}\calE_0(t+1)+\frac{u}{n_0}+\calE^p_0(t)
\end{align*}
where the first inequality is by $\hat Q_t=\hat\delta_t+\hat Q_t^p$ and the triangle inequality, and the last inequality applies Lemma \ref{lemma:L2-Ln-relation} on $\hat\calE^p_0(t)$.

Note that $v_{n_0}^2=\frac{JN_2^2L_2^2\log(N_2L_2)\log n_0}{n_0}$. Set the neural network size such that $N_2L_2\asymp n_0^{\frac{1}{4\gamma_2+2}}$, we get with probability at least $1-4e^{-u}$,
\[\calE_0(t)=\|\hat Q_t-Q_t^*\|^2_{n_0,\PP_t^{(0)}}\lesssim (\frac{J\log n_0}{n_0})^{\frac{2\gamma_2}{2\gamma_2+1}}+\frac{\gamma^2}{\underline{c}}\calE_0(t+1)+\frac{u}{n_0}+\calE^p_0(t).\]
\end{proof}

\textbf{Proof of Theorem 8
}
\begin{proof}
From Lemma \ref{lemma:calEp-bound}, \ref{lemma:calE0-bound} and the union bound, we have with probability at least $1-7Te^{-u}$, for every $t\in[T]$ it holds that
\begin{align*}
\calE^p(t)\lesssim \left(\frac{J\log n_\calM}{n_\calM}\right)^{\frac{2\gamma_1}{2\gamma_1+1}}+\gamma^2 \hat\Omega(t)+\frac{\gamma^2 \Upsilon^2}{\underline{c}}\calE(t+1)+\frac{u}{n_\calM},
\end{align*}
\[\calE_0(t)\lesssim \paren{\frac{J\log n_0}{n_0}}^{\frac{2\gamma_2}{2\gamma_2+1}}+\frac{\gamma^2}{\underline{c}}\calE_0(t+1)+\frac{u}{n_0}+\calE^p_0(t).\]

By Assumption 6
, we have that 
\[\calE_0(t)=\|\hat Q_t-Q^*_t\|^2_{2,\PP_{t}^{(0)}}\ge \eta \|\hat Q_t-Q^*_t\|^2_{2,\PP_{t}^{\operatorname{agg}}}=\eta \calE(t).\]

Similarly, we have that $\calE(t)\ge \eta \calE_0(t)$, and $\eta \calE_0^p(t)\le\calE^p(t)\le \frac{1}{\eta} \calE_0^p(t)$.

Therefore, we can recursively bound $\calE_0(t)$ as 
\begin{align*}
\calE_0(t)&\lesssim (\frac{J\log n_0}{n_0})^{\frac{2\gamma_2}{2\gamma_2+1}}+\frac{\gamma^2}{\underline{c}}\calE_0(t+1)+\frac{u}{n_0}+\calE^p_0(t)\\
&\lesssim (\frac{J\log n_0}{n_0})^{\frac{2\gamma_2}{2\gamma_2+1}}+\frac{\gamma^2}{\underline{c}}\calE_0(t+1)+\frac{u}{n_0}+\frac{1}{\eta}\calE^p(t)\\
&\lesssim (\frac{J\log n_0}{n_0})^{\frac{2\gamma_2}{2\gamma_2+1}}+\frac{\gamma^2}{\underline{c}}\calE_0(t+1)+\frac{u}{n_0}+\frac{1}{\eta}\left(\frac{J\log n_\calM}{n_\calM}\right)^{\frac{2\gamma_1}{2\gamma_1+1}}+\frac{\gamma^2}{\eta}\hat\Omega(t)+\frac{\gamma^2 \Upsilon^2}{\underline{c}\eta}\calE(t+1)+\frac{u}{n_\calM\eta}\\
&\lesssim (\frac{J\log n_0}{n_0})^{\frac{2\gamma_2}{2\gamma_2+1}}+\frac{\gamma^2}{\underline{c}}\calE_0(t+1)+\frac{u}{n_0}+\frac{1}{\eta}\left(\frac{J\log n_\calM}{n_\calM}\right)^{\frac{2\gamma_1}{2\gamma_1+1}}+\frac{\gamma^2}{\eta}\hat\Omega(t)+\frac{\gamma^2 \Upsilon^2}{\underline{c}\eta^2}\calE_0(t+1)+\frac{u}{n_\calM\eta}\\
&\lesssim (\frac{J\log n_0}{n_0})^{\frac{2\gamma_2}{2\gamma_2+1}}+\frac{1}{\eta}\left(\frac{J\log n_\calM}{n_\calM}\right)^{\frac{2\gamma_1}{2\gamma_1+1}}+\frac{\gamma^2}{\eta}\hat\Omega(t)+\kappa\calE_0(t+1)+\left(\frac{u}{n_0}+\frac{u}{n_\calM \eta}\right)
\end{align*}
where recall that $\kappa=\left(\frac{\gamma^2}{\underline{c}}+\frac{\gamma^2 \Upsilon^2}{\underline{c}\eta^2}\right)$.

As $\calE_0(T+1)=0$, we can iteratively get 
\[\calE_0(t)\lesssim (T-t)\max\{\kappa,1\}^{T-t}\left((\frac{J\log n_0}{n_0})^{\frac{2\gamma_2}{2\gamma_2+1}}+\frac{1}{\eta}\left(\frac{J\log n_\calM}{n_\calM}\right)^{\frac{2\gamma_1}{2\gamma_1+1}}+\frac{\gamma^2 T^2}{\eta}\max_{t\le\tau\le T}\hat\Omega(\tau)+\frac{u}{\min(n_0,n_\calM \eta)}\right).\]

\end{proof}

\section{Transition Ratio Estimation by DNN} \label{append:transition-ratio-estimation}
This section is devoted to establishing density estimation error bound, as well as discussing the density transfer.
Recall the definition of the estimator, 
\begin{align*}
\hat\rho_t^{(k)}:&=\arg\min_{g\in \calG(\bar L,\bar N,\bar M,\bar B)} \frac{1}{2n_k}\sum_{i=1}^{n_k} g(\bs^{(k)}_{t,i},a^{(k)}_{t,i},\bs^\circ_i)^2-\frac{1}{n_k}\sum_{i=1}^{n_k}g(\bs^{(k)}_{t,i},a^{(k)}_{t,i},\bs'^{(k)}_{t,i})
\end{align*}

\subsection{Proof of the First result in Theorem 9
} \label{append:proof-thm:transition-density-estimation-error}

\begin{proof}[Proof of the first result in Theorem 9
]
The proof involves the localization analysis on this loss function $g(\bs^{(k)}_{t,i},a^{(k)}_{t,i},\bs^\circ_i)^2-\frac{1}{2}g(\bs^{(k)}_{t,i},a^{(k)}_{t,i},\bs'^{(k)}_{t,i})$. To lighten notation, we omit $k,t$ as they are fixed throughout the proof.

We first state the following two variants of Lemma \ref{lemma:L2-Ln-relation}.

\begin{lemma}
\label{lemma:L2-Ln-relation-version2}
Let $z_1,\cdots,z_n\in \calZ$ be i.i.d. copies of $\bz$, $b\asymp 1$ and $\calG$ be a $b$-uniformly-bounded function class satisfying $\log(\calN_\infty(\epsilon,\calG,\bz_1^n))\le v\log\left(\frac{ebn}{\epsilon}\right)$ for some quantity $v\in (0,1)$. Then for any constant $\zeta\in (0,1)$, there exists constants $c_1,c_2,c_3$ such that as long as $t\ge c_1\sqrt{\frac{v\log n}{n}}$, with probability at least $1-c_2 e^{-c_3nt^2}$, we have
\[\Big|\frac{1}{n}\sum_{i=1}^n g(z_i)-\EE[g(z_1)]\Big|\le \zeta(\EE|g(z_1)|^2+t^2),\quad \forall g\in \calG.\]
\end{lemma}

\begin{proof}[Proof of Lemma \ref{lemma:L2-Ln-relation-version2}]
Let  $v_n=\sqrt{\frac{v\log n}{n}}$. For $\calB(r,\calG):=\{g\in \calG:\|g\|_2=\sqrt{\EE|g(z_1)|^2}\le r\}$, $r\ge v_n$, we can conduct symmetrization as 
\[\EE\sup_{g\in \calB(r,\calG)}\Big|\frac{1}{n}\sum_{i=1}^n g(z_i)-\EE[g(z_1)]\Big|\le 2\EE\Big[\sup_{g\in \calB(r,\calG)}\Big|\frac{1}{n}\sum_{i=1}^n \epsilon_i g(z_i)\Big|\Big]\]where $\epsilon_i$ are i.i.d. Rademacher variables. 
By applying Lemma \ref{lemma:L2-Ln-relation}, we have with probability at least $1-e^{-cnt^2}$, we have that  $\calB(r,\calG)\subset \calB_n(2r+t,\calG,\bz_1^n):=\{g\in \calG:\|g\|_n\le r\}$.
Applying chaining we have that 
\begin{align*}
\EE\Big[\sup_{g\in \calB(r,\calG)}\Big|\frac{1}{n}\sum_{i=1}^n \epsilon_i g(z_i)\Big|\Big]&\lesssim \frac{1}{\sqrt{n}}\EE\int_{0}^b \sqrt{\log\calN_n(\epsilon,\calB(r,\calG),\bz_1^n)}d\epsilon\\&\lesssim \frac{1}{\sqrt{n}}\EE\int_{0}^{2r+t} \sqrt{\log\calN_n(\epsilon,\calB_n(2r+t,\calG,\bz_1^n),\bz_1^n)}d\epsilon\\&\lesssim \frac{1}{\sqrt{n}}\EE\int_{0}^{2r+t} \sqrt{\log(\calN_\infty(\epsilon,\calG,\bz_1^n))}d\epsilon\\&\lesssim (2r+t)\sqrt{\frac{v\log n}{n}}\lesssim r^2+t^2+\frac{v\log n}{n}\end{align*}
where in the second inequality we used that for $\epsilon>2r+t$, $\calN_n(\epsilon,\calB_n(2r+t,\calG,\bz_1^n),\bz_1^n)=1$.

Therefore, by Talagrand's concentration (Theorem 3.27 in \cite{wainwright2019high}), we have for $t\ge c_1\sqrt{\frac{v\log n}{n}}$, there exist some $c_1,c_2,\tilde c_3$ with probability at least $1-c_2e^{-\tilde c_3nt^2}$,
\[\sup_{g\in \calB(r,\calG)}\Big|\frac{1}{n}\sum_{i=1}^n g(z_i)-\EE[g(z_1)]\Big|\le \frac{\zeta}{4}r^2+t^2.\]

We now use the peeling argument to extend to uniform $r$. That is, define $\calS_m=\{g\in \calG:2^m v_n\le\sqrt{\EE|g(z_1)|^2}\le 2^{m+1}v_n\}$. We have 
\[\Big|\frac{1}{n}\sum_{i=1}^n g(z_i)-\EE[g(z_1)]\Big|\le \frac{\zeta}{4}(2^{m+1}v_n)^2+t^2\le\zeta\EE|g(z_1)|^2+t^2\]for every $g\in\calS_m$. A union bound then indicates that with probability at least $1-c_2(\log n) e^{-\tilde c_3nt^2}$, the above holds for every $1\le m\lesssim \log n$, and hence for every $g\in \calG$. Note that $t\ge v_n\ge \sqrt{\frac{\log n}{n}}$, we have $e^{nt^2}\gtrsim n\gtrsim \log n$, let $c_3=\tilde c_3+1$ we can remove the $\log n$ coming from the union bound in the probability term.
\end{proof}

\begin{lemma}
\label{lemma:L2-Ln-relation-version3}
Let $z_1,\cdots,z_n\in \calZ$ be i.i.d. copies of $\bz$, $b\asymp 1$ and  $\calG$ be a $b$-uniformly-bounded function class satisfying $\log(\calN_\infty(\epsilon,\calG,\bz_1^n))\le v\log\left(\frac{ebn}{\epsilon}\right)$ for some quantity $v$. Let $\tilde g$ be a fixed $b$-uniformly-bounded function, not necessarily in $\calG$. Then for any constant $\zeta\in (0,1)$, there exists $c_1,c_2,c_3$ such that as long as $t\ge c_1\sqrt{\frac{v\log n}{n}}$, with probability at least $1-c_2 e^{-c_3nt^2}$, we have
\[\Big|\frac{1}{n}\sum_{i=1}^n (g^2(z_i)-\tilde g^2(z_i))-\EE[g^2(z_1)-\tilde g^2(z_1)]\Big|\le \zeta(\EE|g(z_1)-\tilde g(z_1)|^2+t^2),\quad \forall g\in \calG.\]
\end{lemma}

\begin{proof}[Proof of Lemma \ref{lemma:L2-Ln-relation-version3}]
Define a new function class as $\bar\calG=\{g^2-\tilde g^2:g\in \calG\}$. For $g_1,\cdots,g_N$ being an $\epsilon$-covering set of $\calG$, we claim that $g_1^2-\tilde g^2,\cdots,g_N^2-\tilde g^2$ is an $2b\epsilon$-covering set of $\bar \calG$. In fact, for any $g\in\calG$, there exists $\pi(g)\in [N]$ such that $|g(z_i)-g_{\pi(g)}(z_i)|\le \epsilon$. Therefore, $|(g(z_i)^2-\tilde g(z_i)^2)-(g_{\pi(g)}(z_i)^2-\tilde g(z_i)^2)|=|(g_{\pi(g)}(z_i)-g(z_i))\cdot (g_{\pi(g)}(z_i)+g(z_i))|\le 2b\epsilon$. And hence $\log(\calN_\infty(\epsilon,\bar\calG,\bz_1^n))\le v\log\left(\frac{2eb^2 n}{\epsilon}\right)\le 2v\log\left(\frac{eb n}{\epsilon}\right)$.

Applying Lemma \ref{lemma:L2-Ln-relation-version2} on $\bar \calG$ with $\zeta$ replaced by $\frac{\zeta}{4b^2}$, we have with probability at least $1-c_2 e^{-c_3nt^2}$,
\[\Big|\frac{1}{n}\sum_{i=1}^n (g^2(z_i)-\tilde g^2(z_i))-\EE[g^2(z_1)-\tilde g^2(z_1)]\Big|\le \frac{\zeta}{4b^2}(\EE|g^2(z_1)-\tilde g^2(z_1)|^2+t^2),\quad \forall g\in \calG.\]
Noticing $\EE|g^2(z_1)-\tilde g^2(z_1)|^2\le 4b^2\EE|g(z_1)-\tilde g(z_1)|^2$ completes the proof.

\end{proof}

We first define empirical loss $\hat J$ and the population version 
loss $J$ as 
\[\hat J(g)=\frac{1}{2n}\sum_{i=1}^n g(\bs_i,a_i,\bs_i^\circ)^2-\frac{1}{n}\sum_{i=1}^n g(\bs_i,a_i,\bs'_i).\]
\[J(g)=\frac{1}{2}\EE g(\bs_i,a_i,\bs_i^\circ)^2-\EE g(\bs_i,a_i,\bs'_i).\]
We have by change of variable from $\bs'_i$ to $\bs_i^\circ$,
\begin{align*}
J(g)&=\frac{1}{2}\int\Big[ g(\bs_i,a_i,\bs_i^\circ)^2-2g(\bs_i,a_i,\bs^\circ_i)\rho(\bs_i,a_i,\bs_i^\circ)\Big]p(\bs_i,a_i)d\bs_i da d\bs_i^\circ
\end{align*}
Therefore, we have 
\[J(g)-J(\rho)=\frac{1}{2}\int \left(g(\bs_i,a_i,\bs_i^\circ)-\rho(\bs_i,a_i,\bs_i^\circ)\right)^2p(\bs_i,a_i)d\bs_i da d\bs_i^\circ.\]

Again using neural network approximation results \citep{fan2023factor}, we have a $\bar g\in \calG(\bar L,\bar N,\bar M,\bar B)$, such that $\|\bar g-\rho\|_\infty\le (\bar N\bar L)^{-2\gamma_3}$.

The optimality of $\hat \rho$ leads to
\[\hat J(\hat \rho)\le \hat J(\bar g).\]

By bounding the metric entropy of $\calG$ similar as in Lemma \ref{lemma:calE0-bound}, the conditions in Lemma \ref{lemma:L2-Ln-relation-version2} and \ref{lemma:L2-Ln-relation-version3} are satisfied with $v=\bar N^2\bar L^2\log(\bar N\bar L)$.

Applying Lemma \ref{lemma:L2-Ln-relation-version3} we have with probability at least $1-c_2e^{-c_3nu^2}$, and $u\ge c_1 v_n$,
\[\Big|\frac{1}{2n}\sum_{i=1}^n \bar g(\bs_i,a_i,\bs_i^\circ)^2-\frac{1}{2n}\sum_{i=1}^n \hat \rho(\bs_i,a_i,\bs_i^\circ)^2-\frac{1}{2}\EE \bar g(\bs_i,a_i,\bs_i^\circ)^2 +\frac{1}{2}\EE \hat \rho(\bs_i,a_i,\bs_i^\circ)^2 \Big|\le\frac{1}{8\Upsilon_2}(\EE(\bar g-\hat\rho)^2+u^2).\]

Applying Lemma \ref{lemma:L2-Ln-relation-version2}, we have with probability at least $1-c_2e^{-c_3nu^2}$, and $u\ge c_1 v_n$,
\[\Big|\frac{1}{n}\sum_{i=1}^n \bar g(\bs_i,a_i,\bs'_i)-\frac{1}{n}\sum_{i=1}^n \hat\rho(\bs_i,a_i,\bs'_i)-\EE\bar g(\bs_i,a_i,\bs'_i)+\EE\hat\rho(\bs_i,a_i,\bs'_i)\Big|\le \frac{1}{8\Upsilon_2}(\EE(\bar g-\hat\rho)^2+u^2).\]
Combining the above two inequalities we obtain $|\hat J(\bar g)-\hat J(\hat\rho)-J(\bar g)+J(\hat\rho)|\le\frac{1}{4\Upsilon_2}(\EE(\bar g-\hat\rho)^2+u^2)$.

For the difference of population loss, we have \begin{align*}
J(\bar g)-J(\hat\rho)&=J(\bar g)-J(\rho)+J(\rho)-J(\hat\rho)\\
&\le \frac{1}{2}(\bar N\bar L)^{-4\gamma_3}+J(\rho)-J(\hat\rho)\end{align*}

Therefore, we have 
\begin{align*}
J(\hat\rho)-J(\rho)&\le\frac{1}{4\Upsilon_2}(\EE(\bar g-\hat\rho)^2+u^2)+ \frac{1}{2}(\bar N\bar L)^{-4\gamma_3}\\
&\le \frac{1}{4\Upsilon_2}(\EE(\rho-\hat\rho)^2+u^2)+ (\bar N\bar L)^{-4\gamma_3}
\end{align*}
where we use approximation results in the second inequality again and $\Upsilon_2\ge 1$.

On the other hand, \begin{align*}J(\hat\rho)-J(\rho)&=\frac{1}{2}\int \left(\hat\rho(\bs_i,a_i,\bs_i^\circ)-\rho(\bs_i,a_i,\bs_i^\circ)\right)^2p(\bs_i,a_i)d\bs_i da d\bs_i^\circ\\
&=\frac{1}{2}\int \left(\hat\rho(\bs_i,a_i,\bs'_i)-\rho(\bs_i,a_i,\bs'_i)\right)^2\frac{p(\bs_i,a_i,\bs'_i)}{\rho(\bs_i,a_i,\bs'_i)}d\bs_i da d\bs'_i\\&\ge \frac{1}{2\Upsilon_2}\int \left(\hat\rho(\bs_i,a_i,\bs'_i)-\rho(\bs_i,a_i,\bs'_i)\right)^2p(\bs_i,a_i,\bs'_i)d\bs_i da d\bs'_i\\&=\frac{1}{2\Upsilon_2}\EE(\rho-\hat\rho)^2
\end{align*}
Putting pieces together, we have for $u\ge \sqrt{\frac{\bar N^2\bar L^2\log(\bar N\bar L)\log n}{n}}$, with probability at least $1-c_2e^{-c_3nu^2}$,
\[\EE(\rho-\hat\rho)^2\lesssim u^2+ \Upsilon_2(\bar N\bar L)^{-4\gamma_3}.\]
That is equivalent to saying that with probability at least $1-e^{-u}$,
\[\EE(\rho-\hat\rho)^2\lesssim \frac{u}{n}+\Upsilon_2(\bar N\bar L)^{-4\gamma_3}+\frac{\bar N^2\bar L^2\log(\bar N\bar L)\log n}{n}.\]
Applying Lemma \ref{lemma:L2-Ln-relation} again and adding back scripts $k,t$, we have that 
\[\max\Big\{\EE^{(k)}(\rho_{t,i}^{(k)}-\hat\rho_{t,i}^{(k)})^2,\frac{1}{n_k}\sum_{i=1}^{n_k}(\rho_{t,i}^{(k)}-\hat\rho_{t,i}^{(k)})^2\Big\}\lesssim \frac{u}{n_k}+\Upsilon_2(\bar N\bar L)^{-4\gamma_3}+\frac{\bar N^2\bar L^2\log(\bar N\bar L)\log n_k}{n_k}.\]
Set the size parameters such that $\bar N\bar L\asymp (\frac{n_k\Upsilon_2}{\log n_k})^{\frac{1}{4\gamma_3+2}}$, we have that with probability at least $1-e^{-u}$,
\[\max\Big\{\EE^{(k)}(\rho_{t}^{(k)}-\hat\rho_{t}^{(k)})^2,\frac{1}{n_k}\sum_{i=1}^{n_k}(\rho_{t,i}^{(k)}-\hat\rho_{t,i}^{(k)})^2\Big\}\lesssim \frac{u}{n_k}+(\frac{\log n_k}{n_k})^{\frac{2\gamma_3}{2\gamma_3+1}}\]where recall that $\EE^{(k)}$ means data generating process under task $k$.
\end{proof}

\subsection{Proof of the Second Result in Theorem 9
}

\begin{proof}[Proof of the second result in Theorem 9
]
From the first reuslt in Theorem 9 
and a union bound, with probability at least $1-(K+1)e^{-u}$, for every $k=0,1,\cdots,K$,
\begin{equation}
\label{equ:transition-density-estimation-bound}
\max\Big\{\EE^{(k)}(\rho_{t}^{(k)}-\hat\rho_{t}^{(k)})^2,\frac{1}{n_k}\sum_{i=1}^{n_k}(\rho_{t,i}^{(k)}-\hat\rho_{t,i}^{(k)})^2\Big\}\lesssim \frac{u}{n_k}+(\frac{\log n_k}{n_k})^{\frac{2\gamma_3}{2\gamma_3+1}}\end{equation}
By Assumption 6
, we have for $1\le k\le K$,
\[\EE^{(k)}(\rho_{t}^{(0)}-\hat\rho_{t}^{(0)})^2\le\frac{1}{\eta}\EE^{(k)}(\rho_{t}^{(0)}-\hat\rho_{t}^{(0)})^2\lesssim \frac{u}{n_0\eta}+\frac{1}{\eta}(\frac{\log n_0}{n_0})^{\frac{2\gamma_3}{2\gamma_3+1}} .\]
Using Lemma \ref{lemma:L2-Ln-relation}, we have for $1\le k\le K$, 
\begin{equation}
\label{equ:transition-density-estimation-bound1}
\frac{1}{n_k}\sum_{i=1}^{n_k}(\rho_{t}^{(0)}(\bs_{t,i}^{(k)},a_{t,i}^{(k)},\bs'^{(k)}_{t,i})-\hat\rho_{t}^{(0)}(\bs_{t,i}^{(k)},a_{t,i}^{(k)},\bs'^{(k)}_{t,i}))^2\lesssim \frac{u}{n_0\eta}+\frac{1}{\eta}(\frac{\log n_0}{n_0})^{\frac{2\gamma_3}{2\gamma_3+1}}+\frac{u}{n_k}+(\frac{\log n_k}{n_k})^{\frac{2\gamma_3}{2\gamma_3+1}} \end{equation}
Note that we are interested in bounding $\hat\Omega(t)=\frac{1}{n_\calM}\sum_{i=1}^{n_\calM}|\hat \omega_{t,i}^{k_i}-\omega_{t,i}^{k_i}|^2=\frac{1}{n_\calM}\sum_{k=1}^K\hat\Omega_k(t)$, where $\hat\Omega_k(t)=\sum_{i=1}^{n_k}|\hat \omega_{t,i}^{k}-\omega_{t,i}^{k}|^2$.

We have by Assumption 4
(i) and the truncation step at $\Upsilon_1$,
\begin{align*}|\hat \omega_{t,i}^{k}-\omega_{t,i}^{k}|^2&=\Big|\frac{\rho_{t}^{(0)}(\bs_{t,i}^{(k)},a_{t,i}^{(k)},\bs'^{(k)}_{t,i})}{\max\{\rho_{t,i}^{(k)},\Upsilon_1\}}-\frac{\hat\rho_{t}^{(0)}(\bs_{t,i}^{(k)},a_{t,i}^{(k)},\bs'^{(k)}_{t,i})}{\max\{\hat\rho_{t,i}^{(k)},\Upsilon_1\}}\Big|^2\\
&\lesssim (\rho_{t}^{(0)}(\bs_{t,i}^{(k)},a_{t,i}^{(k)},\bs'^{(k)}_{t,i})-\hat\rho_{t}^{(0)}(\bs_{t,i}^{(k)},a_{t,i}^{(k)},\bs'^{(k)}_{t,i}))^2+ (\rho_{t,i}^{(k)}-\hat\rho_{t,i}^{(k)})^2
\end{align*}
And hence 
\[\hat\Omega_k(t)\lesssim n_k\left(\frac{u}{n_0\eta}+\frac{1}{\eta}(\frac{\log n_0}{n_0})^{\frac{2\gamma_3}{2\gamma_3+1}}+\frac{u}{n_k}+(\frac{\log n_k}{n_k})^{\frac{2\gamma_3}{2\gamma_3+1}}\right).\]
Summing up we get with probability at least $1-T(K+1)e^{-u}$, for every $t\in [T]$,
\begin{align*}
\hat\Omega(t)&\lesssim \frac{1}{n_\calM}\sum_{k=1}^K n_k\left(\frac{u}{n_0\eta}+\frac{1}{\eta}(\frac{\log n_0}{n_0})^{\frac{2\gamma_3}{2\gamma_3+1}}+\frac{u}{n_k}+(\frac{\log n_k}{n_k})^{\frac{2\gamma_3}{2\gamma_3+1}}\right) \\
&\lesssim \frac{u}{n_0\eta}+\frac{Ku}{n_\calM}+\frac{1}{\eta}(\frac{\log n_0}{n_0})^{\frac{2\gamma_3}{2\gamma_3+1}}+\log^{\frac{2\gamma_3}{2\gamma_3+1}}(n_\calM)\frac{\sum_{k=1}^K n_k^{\frac{1}{2\gamma_3+1}}}{n_\calM}\\
&\le \frac{u}{n_0\eta}+\frac{Ku}{n_\calM}+\frac{1}{\eta}(\frac{\log n_0}{n_0})^{\frac{2\gamma_3}{2\gamma_3+1}}+\log^{\frac{2\gamma_3}{2\gamma_3+1}}(n_\calM)\frac{K^{\frac{2\gamma_3}{2\gamma_3+1}} n_{\calM}^{\frac{1}{2\gamma_3+1}}}{n_\calM}\\
&\lesssim \frac{u}{\min\{n_0,n_\calM/K\}}+\frac{1}{\eta}(\frac{\log n_0}{n_0})^{\frac{2\gamma_3}{2\gamma_3+1}}+(\frac{K\log(n_\calM)}{n_\calM})^{\frac{2\gamma_3}{2\gamma_3+1}}
\end{align*}
\end{proof}

\section{Transition Ratio Estimation by DNN with Density Transfer} \label{append:transition-ratio-estimation-density-transfer}

Recall that the density ratio estimator under density similarity is given by 
\begin{align*}
\hat\omega_t^{(k)}:&=\arg\min_{g\in \calG(\bar L_2,\bar N_2,\bar M_2,\bar B_2)} \frac{1}{2n_0}\sum_{i=1}^{n_0} (g\cdot \hat\rho_t^{(k)})^2(\bs^{(0)}_{t,i},a^{(0)}_{t,i},\bs^\circ_i)-\frac{1}{n_0}\sum_{i=1}^{n_0} (g\cdot\hat\rho_t^{(k)})(\bs^{(0)}_{t,i},a^{(0)}_{t,i},\bs'^{(0)}_{t,i})
\end{align*}

\subsection{Proof of the First Result in Theorem 10}

\begin{proof}[Proof of the first result in Theorem 10]
To lighten notation, we omit subscript $t$ and superscripts as they are fixed through the proof, except keeping superscript $(k)$ in $\hat\rho^{(k)}$. For example, we abbreviate $\hat\omega_t^{(k)}$ as $\hat\omega$ and abbreviate $\hat\rho^{(k)}_t$ as $\hat\rho^{(k)}$.

For any $g$, define
\[\hat J_0(g):=\frac{1}{2n}\sum_{i=1}^n g(\bs_i,a_i,\bs_i^\circ)^2-\frac{1}{n}\sum_{i=1}^n g(\bs_i,a_i,\bs'_i).\]
\[J_0(g):=\frac{1}{2}\EE g(\bs_i,a_i,\bs_i^\circ)^2-\EE g(\bs_i,a_i,\bs'_i)\]where note the tuples $(\bs_i,a_i,\bs'_i)$ come from the target task. We have
\[J_0(g)-J_0(\omega\rho^{(k)})=\frac{1}{2}\int \left(g(\bs_i,a_i,\bs_i^\circ)-(\omega\cdot\rho^{(k)})(\bs_i,a_i,\bs_i^\circ)\right)^2p(\bs_i,a_i)d\bs_i da d\bs_i^\circ.\]
Employing neural network approximation results \citep{fan2023factor} and by Assumption 4, we have a $\bar g\in \calG(\bar L_2,\bar N_2,\bar M_2,\bar B_2)$, such that $\|\bar g-\omega\|_\infty\le (\bar N_2\bar L_2)^{-2\gamma_4}$.

The optimality of the estimator leads to
\[\hat J_0(\hat \omega\hat\rho^{(k)})\le \hat J_0(\bar g\hat\rho^{(k)}).\]

By bounding the metric entropy of $\calG$ similar as in Lemma \ref{lemma:calE0-bound}, the conditions in Lemma \ref{lemma:L2-Ln-relation-version2} and \ref{lemma:L2-Ln-relation-version3} are satisfied with $v=\bar N_2^2\bar L_2^2\log(\bar N_2\bar L_2)$.

Similar to the proof of Theorem 10, applying Lemma \ref{lemma:L2-Ln-relation-version2} and \ref{lemma:L2-Ln-relation-version3} and adding up, we obtain with probability at least $1-c_2e^{-c_3nu^2}$, \begin{align*}|\hat J_0(\bar g\hat\rho^{(k)})-\hat J_0(\hat \omega\hat\rho^{(k)})-J_0(\bar g\hat\rho^{(k)})+J_0(\hat \omega\hat\rho^{(k)})|\le\frac{1}{8}(\EE_0(\bar g\hat\rho^{(k)}-\hat\omega\hat\rho^{(k)})^2+u^2).
\end{align*}

Moreover, We have the decomposition
\begin{align*}
J_0(\bar g\hat\rho^{(k)})-J_0(\hat\omega\hat\rho^{(k)})
= J_0(\bar g\hat\rho^{(k)})-J_0( \omega\rho^{(k)})+J_0( \omega\rho^{(k)})-J_0(\hat\omega\hat\rho^{(k)}),
\end{align*}
and
\begin{align*}|J_0(\bar g\hat\rho^{(k)})-J_0( \omega\rho^{(k)})|&\le\frac{1}{2}\int \left((\bar g\cdot\hat\rho^{(k)})(\bs_i,a_i,\bs_i^\circ)-(\omega\cdot\rho^{(k)})(\bs_i,a_i,\bs_i^\circ)\right)^2p(\bs_i,a_i)d\bs_i da d\bs_i^\circ \\
&\le \frac{\Upsilon_2}{2}\EE_0|\bar g\cdot\hat\rho^{(k)}-\omega\cdot\rho^{(k)}|^2\\
&\le \frac{\Upsilon_2}{2}(\EE_0|\bar g\cdot\hat\rho^{(k)}-\bar g\cdot\rho^{(k)}|^2+\EE_0|\bar g\cdot\rho^{(k)}-\omega\cdot\rho^{(k)}|^2)\\
&\le \frac{\Upsilon_2}{2}(\bar B_2^2\EE_0|\hat\rho^{(k)}-\rho^{(k)}|^2+\Upsilon_2^2\EE_0|\bar g-\omega|^2)\\
&\lesssim \EE_0|\hat\rho^{(k)}-\rho^{(k)}|^2+(\bar N_2\bar L_2)^{-2\gamma_4}
\end{align*}
where we used the boundedness of $\bar g$ and $\rho^{(k)}$, and $\EE_0$ means the expectation is taken in target samples.

Putting pieces together, we get that for some constant $C$,
\begin{align*}
\EE_0|\omega\rho^{(k)}-\hat\omega\hat\rho^{(k)}|^2&=J_0( \omega\rho^{(k)})-J_0( \hat\omega\hat\rho^{(k)}) \\
&\le \frac{1}{8}(\EE_0(\bar g\hat\rho^{(k)}-\hat\omega\hat\rho^{(k)})^2+u^2)+C\EE_0|\hat\rho^{(k)}-\rho^{(k)}|^2+C(\bar N_2\bar L_2)^{-2\gamma_4}\\
&\le \frac{1}{2}(\EE_0(\omega\hat\rho^{(k)}-\hat\omega\hat\rho^{(k)})^2+u^2)+C\EE_0|\hat\rho^{(k)}-\rho^{(k)}|^2+2C(\bar N_2\bar L_2)^{-2\gamma_4}
\end{align*}
Meanwhile, we have \begin{align*}\EE_0|\omega\rho^{(k)}-\hat\omega\hat\rho^{(k)}|^2&\ge \frac{1}{2}\EE_0(\omega\hat\rho^{(k)}-\hat\omega\hat\rho^{(k)})^2-\EE_0(\omega\hat\rho^{(k)}-\omega\rho^{(k)})^2\\
&\ge \frac{1}{2}\EE_0(\omega\hat\rho^{(k)}-\hat\omega\hat\rho^{(k)})^2-\frac{\Upsilon_2^2}{\Upsilon_1^2}\EE_0(\hat\rho^{(k)}-\rho^{(k)})^2\end{align*}where in the last inequality we used $|\omega|\le \frac{\Upsilon_2}{\Upsilon_1}$.

As a result, we have with probability at least $1-e^{-u}$,
\[\EE_0(\omega\hat\rho^{(k)}-\hat\omega\hat\rho^{(k)})^2\lesssim \EE_0|\hat\rho^{(k)}-\rho^{(k)}|^2+(\bar N_2\bar L_2)^{-2\gamma_4}+\frac{\bar N_2^2\bar L_2^2\log n_0\log (\bar N_2\bar L_2)}{n_0}+\frac{u}{n_0}.\]
Letting $\bar N_2\bar L_2\asymp (\frac{n_0}{\log n_0})^{\frac{1}{4\gamma_4+2}}$, as well as using $|\hat\rho^{(k)}|\ge \Upsilon_1$ by the truncation step, we have that with probability at least $1-e^{-u}$, \[\EE_0(\omega_t^{(k)}-\hat\omega_t^{(k)})^2\lesssim\EE_0|\hat\rho_t^{(k)}-\rho_t^{(k)}|^2+ (\frac{\log n_0}{n_0})^{\frac{2\gamma_4}{2\gamma_4+1}}+\frac{u}{n_0}.\]
\end{proof}

\subsection{Proof of the Second Result in Theorem 10}

\begin{proof}[Proof of the second result in Theorem 10]
From Theorem 9, with probability at least $1-e^{-u}$,
\[\EE^{(k)}(\rho_{t}^{(k)}-\hat\rho_{t}^{(k)})^2\lesssim \frac{u}{n_k}+(\frac{\log n_k}{n_k})^{\frac{2\gamma_3}{2\gamma_3+1}}.\]    
Therefore, from Theorem 10,
\begin{align*}
\EE^{(k)}(\omega_t^{(k)}-\hat\omega_t^{(k)})^2&\le \frac{1}{\eta}\EE_0(\omega_t^{(k)}-\hat\omega_t^{(k)})^2\\&\lesssim\frac{1}{\eta}\EE_0|\hat\rho_t^{(k)}-\rho_t^{(k)}|^2+ \frac{1}{\eta}(\frac{\log n_0}{n_0})^{\frac{2\gamma_4}{2\gamma_4+1}}+\frac{u}{n_0\eta}\\&\le\frac{1}{\eta^2}\EE^{(k)}|\hat\rho_t^{(k)}-\rho_t^{(k)}|^2+ \frac{1}{\eta}(\frac{\log n_0}{n_0})^{\frac{2\gamma_4}{2\gamma_4+1}}+\frac{u}{n_0\eta}\\
&\le\frac{1}{\eta^2}(\frac{\log n_k}{n_k})^{\frac{2\gamma_3}{2\gamma_3+1}}+ \frac{1}{\eta}(\frac{\log n_0}{n_0})^{\frac{2\gamma_4}{2\gamma_4+1}}+\frac{u}{\min\{n_0\eta,n_k\eta^2\}}
\end{align*}

Applying Lemma \ref{lemma:L2-Ln-relation} we have 
\[\frac{1}{n_k}\sum_{i=1}^{n_k}(\omega_{t,i}^{(k)}-\hat\omega_{t,i}^{(k)})^2\lesssim \frac{1}{\eta^2}(\frac{\log n_k}{n_k})^{\frac{2\gamma_3}{2\gamma_3+1}}+ \frac{1}{\eta}(\frac{\log n_0}{n_0})^{\frac{2\gamma_4}{2\gamma_4+1}}+\frac{u}{\min\{n_0\eta,n_k\eta^2\}}+\frac{n_0^{\frac{1}{2\gamma_4+1}}\log n_k}{n_k}.\]
Therefore,
\begin{align*}
\hat\Omega(t)&=\frac{1}{n_\calM}\sum_{k=1}^K \sum_{i=1}^{n_k}(\omega_{t,i}^{(k)}-\hat\omega_{t,i}^{(k)})^2\\
&\lesssim \frac{1}{\eta^2}(\frac{K\log n_\calM}{n_\calM})^{\frac{2\gamma_3}{2\gamma_3+1}}+ \frac{1}{\eta}(\frac{\log n_0}{n_0})^{\frac{2\gamma_4}{2\gamma_4+1}}+\frac{u}{\min\{n_0\eta,n_\calM\eta^2/K\}}+\frac{n_0^{\frac{1}{2\gamma_4+1}}K\log n_\calM}{n_\calM}
\end{align*}
Taking a union bound over $k,t$ yields the desired result.
\end{proof}

\section{MIMIC-III: Calibrated Sepsis Management Environment} \label{append:mimic3}

\noindent\textbf{State variables.}
The samples of state, action, reward, and next state $\braces{{\bx_{i,t}, a_{i,t}, r_{i,t}, \bx_{i,t+1}}}_{i\in[N], t\in[T_i]}$ are constructed as follows. 
Each patient in the cohort is characterized by a set of $45$ variables, including demographics, vital signs, and laboratory values.
We conduct a dimension reduction using principal component analysis (PCA) and choose the top three principal components (PCs) as our state features, which explain about $98.97\%$ of the total variance. 
\begin{figure}[htbp]
\centering
\includegraphics[width=0.7\textwidth]{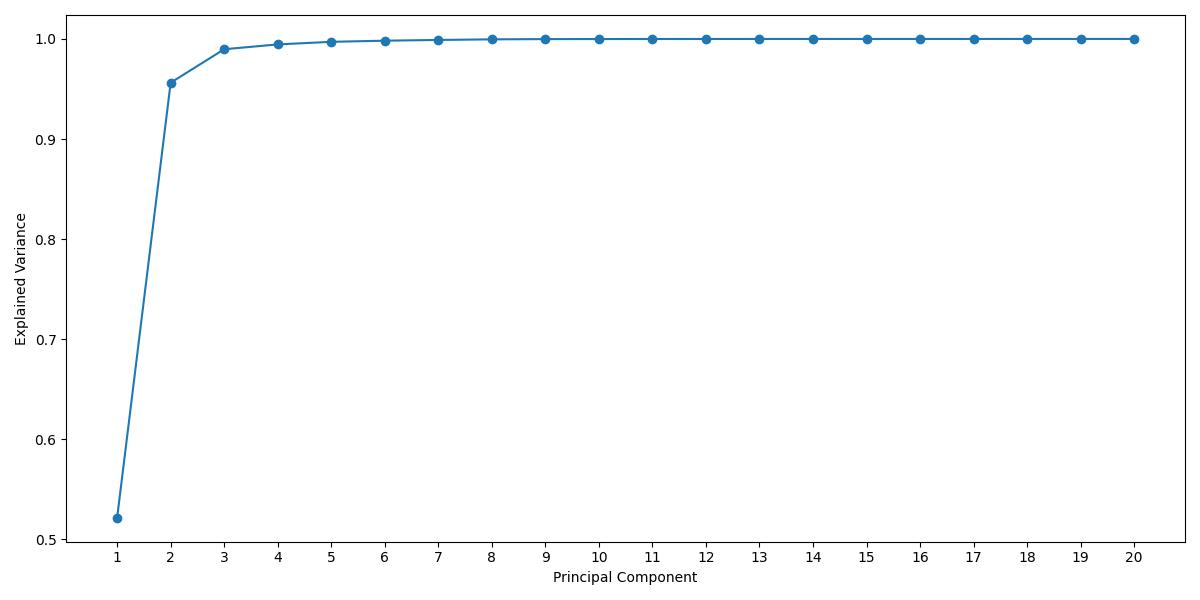}
\caption{Scree plot of the principal component analysis on 45 state variables.}
\label{fig:scree-plot}
\end{figure}

\noindent
\textbf{Rewards.} 
The reward signal is important and is crafted carefully in real applications.
For the final reward, we follow \cite{komorowski2018artificial} and use hospital mortality or 90-day mortality.
Specifically, when a patient survived after 90 days out of hospital, a positive reward (+1) was released at the end of each patient's trajectory; a negative reward (-1) was issued if the patient died in hospital or within 90 days out of hospital.
In our dataset, the mortality rate is 24.21\% for female and 22.71\% for male.
For the intermediate rewards, we follow \cite{prasad2017reinforcement} and associates reward to the health measurement of a patient.
The detailed description of the data pre-processing is presented in Section J of the supplemental material in \cite{chen2022reinforcement}.

\noindent
\textbf{Trajectory horizon and inverse steps.} 
The trajectory horizons are different in the dataset, with the maximum being 20 and the minimum being 1.
The trajectories are aligned at the last steps while allowing the starting steps to vary.
For example, the trajectories with length 20 start at step $1$ while the trajectories with length 10 start at step $11$.
But they all end at step 20.
We adopt this method because the distribution of final status is similar across trajectories.
Figure 5 in \cite{chen2022transferred} presents mortality rates of different lengths.
We see that while the numbers of trajectories differ a lot, the mortality rates do not vary much across trajectories with different horizons.
On the contrary, the starting status of patients may be very different. 
The one with trajectory length $20$ may be in a worse status and needs $10$ steps to reach the status similar to the starting status of the one with length $10$. 
We believe this is a reasonable setup to illustrate our method.
A rigorous medical analysis is beyond the scope of this paper and is a worthwhile topic for future research.  

\noindent
\textbf{Stage Compression.} 
In the calibrated environment with the function class of neural networks, we consider source and target MDPs with 5 stages. 
We have in total 20 steps and we number steps as $0-19$ starting from the end. So we end up having much more samples in the steps close to 0 (end). But we have fewer in the steps close to 19, because most trajectories have fewer than 20 steps, oftentimes fewer than 10 steps.
When generating buckets, we want to group adjacent steps together such that each bucket contains approximately the same number of samples. 
The state variables and rewards for the new aggregated stage are computed by averaging the corresponding values across all original stages that were combined.
The table we use:
\begin{table}[tb]
\centering
\resizebox{0.6\textwidth}{!}{%
\begin{tabular}{c|ccccc}
\hline
Stages (Inclusive) & 0-1 & 2-4 & 5-7 & 8-11 & 12-19 \\ \hline
New Stages & 0 & 1 & 2 & 3 & 4 \\ \hline
\end{tabular}%
}
\caption{Transformation of Stages}
\label{tab:my-table}
\end{table}


\noindent
\textbf{Environment Calibration.} 
We use the final processed data of $26,355$ tuples $\braces{{\bx_{i,t}, a_{i,t}, r_{i,t}, \bx_{i,t+1}}}$ with $i\in [2000]$ and $t \in [5]$. 
We devide the source and target task as corresponding to different gender of the patients. 
We use model (5.2) 
to learn the transition and reward models for the calibrated environments of source and target task respectively. 

\section{Explicit Expression of $Q$ Function in Section 5.2}\label{append:express-Q}

The true coefficients for the Q-functions in (5.1)
are $\theta_{2j} = \kappa_j$, $1\le j\le 7$ and
\begin{equation}  \label{eqn:true-q-theta}
    \begin{aligned}
       \theta_{11} & = \kappa_1 + q_1 \abs{f_1} + q_2 \abs{f_2} + (0.5-q_1) \abs{f_3} + (0.5-q_2)\abs{f_4},  \\
       \theta_{12} & = \kappa_2 + q_1'\abs{f_1} + q_2' \abs{f_2} - q_1' \abs{f_3} - q_2' \abs{f_4}, \\
       \theta_{13} & = \kappa_3 + q_1 \abs{f_1} - q_2 \abs{f_2} + (0.5-q_1) \abs{f_3} - (0.5-q_2)\abs{f_4}, \\
       \theta_{14} & = \kappa_4 + q_1'\abs{f_1} - q_2' \abs{f_2} - q_1' \abs{f_3} + q_2' \abs{f_4},
    \end{aligned}
\end{equation}
where 
\begin{align*}
q_1 & = 0.25\paren{ {\rm expit}\paren{b_1 + b_2} + {\rm expit}\paren{-b_1 + b_2}} \\   
q_2 & = 0.25\paren{ {\rm expit}\paren{b_1 - b_2} + {\rm expit}\paren{-b_1 - b_2} } \\
q_1'& = 0.25\paren{ {\rm expit}\paren{b_1 + b_2} - {\rm expit}\paren{-b_1 + b_2} } \\
q_2'& = 0.25\paren{ {\rm expit}\paren{b_1 - b_2} - {\rm expit}\paren{-b_1 - b_2} } \\
f_1 & = \kappa_5 + \kappa_6 + \kappa_7 \\
f_2 & = \kappa_5 + \kappa_6 - \kappa_7 \\
f_3 & = \kappa_5 - \kappa_6 + \kappa_7 \\
f_4 & = \kappa_5 - \kappa_6 - \kappa_7
\end{align*}

\end{appendices}

\end{document}